\documentclass[lettersize,journal]{IEEEtran}
\usepackage{amsmath,amsfonts}
\usepackage{algorithmic}
\usepackage{algorithm}
\usepackage{array}
\usepackage[caption=false,font=normalsize,labelfont=sf,textfont=sf]{subfig}
\usepackage{textcomp}
\usepackage{stfloats}
\usepackage{url}
\usepackage{verbatim}
\usepackage{graphicx}
\usepackage{cite}
\usepackage{xcolor}
\usepackage{xspace}
\usepackage{booktabs} 
\usepackage{multirow}
\usepackage{enumitem}
\usepackage{subcaption}
\usepackage{amsmath,amsfonts,bm,amsthm}

\def\eqref#1{equation~\ref{#1}}
\def\1{\bm{1}}

\newtheorem{definition}{Definition}
\newtheorem{assumption}{Assumption}
\newtheorem{lemma}{Lemma}
\newtheorem{proposition}{Proposition}

\def\vmu{{\bm{\mu}}}

\def\va{{\bm{a}}}

\def\vf{{\bm{f}}}

\def\vx{{\bm{x}}}

\def\mA{{\bm{A}}}

\DeclareMathAlphabet{\mathsfit}{\encodingdefault}{\sfdefault}{m}{sl}
\SetMathAlphabet{\mathsfit}{bold}{\encodingdefault}{\sfdefault}{bx}{n}

\def\gI{{\mathcal{I}}}

\makeatletter
\DeclareRobustCommand\onedot{\futurelet\@let@token\@onedot}
\def\@onedot{\ifx\@let@token.\else.\null\fi\xspace}

\def\eg{\emph{e.g}\onedot} 
\def\ie{\emph{i.e}\onedot}

\makeatother

\definecolor{beaublue}{rgb}{0.8, 0.9, 1.0}% {0.9, 0.95, 0.9}
\definecolor{blackish}{rgb}{0.2, 0.2, 0.2}

\usepackage{tcolorbox}

\begin{document}

\title{Test-Time Self-Adaptive Conditioning for Stable Audio-Driven Talking-Head Generation
%Test-Time Self-Adaptive Conditioning for Audio-Driven Talking-Head Generation
}

\author{Zhicheng Zhang$^\dagger$, Lei Wang$^\dagger$, Yu Zhang, Yongsheng Gao
\thanks{
$\dagger$ are co-first authors with equal contribution.

Zhicheng Zhang and Yu Zhang are with the School of Business, University of New South Wales (UNSW), Australia (emails: zhicheng.zhang2@unsw.edu.au; m.yuzhang@unsw.edu.au).

Lei Wang and Yongsheng Gao are with the School of Engineering and Built Environment, Griffith University, Australia (emails: l.wang4@griffith.edu.au; yongsheng.gao@griffith.edu.au).
% Zhicheng Zhang and Lei Wang contributed equally to this work and are co-first authors.
 
% Zhicheng Zhang is with the School of Business, University of New South Wales (UNSW), Australia (email: zhicheng.zhang2@unsw.edu.au).

% Lei Wang is with the School of Engineering and Built Environment, Griffith University, Australia, and with Data61/CSIRO (email: l.wang4@griffith.edu.au).

% Yu Zhang is with the School of Business, University of New South Wales (UNSW), Australia (email: m.yuzhang@unsw.edu.au).

% Yongsheng Gao is with the School of Engineering and Built Environment, Griffith University, Australia (email: yongsheng.gao@griffith.edu.au).
% Yu Zhang and Yongsheng Gao are the corresponding authors.
}}



\maketitle

\begin{abstract}
Audio-driven talking-head generation has achieved remarkable progress with recent models such as AniTalker, FLOAT, and Sonic. Despite their success, most existing approaches rely on a single static reference image to condition the entire video generation process at inference stage. This static conditioning paradigm often creates a mismatch between fixed identity features and dynamically evolving facial motion,  leading to identity drift, temporal inconsistency, and degraded perceptual quality.
We introduce Test-Time Self-Adaptive Conditioning (TT-SAC), a parameter-free inference framework that enables pretrained talking-head generators to adapt their conditioning representations during inference without retraining, gradient updates, or additional supervision. Instead of treating the reference portrait as immutable, TT-SAC composes the generator with its encoder in a feedback loop: the generator’s own outputs are re-encoded to construct a refined conditioning representation that better aligns with the temporal dynamics of the synthesized sequence. A single adaptation step approximates a self-consistent equilibrium of the generative process, stabilizing identity and motion across time.
We further provide theoretical analysis showing that test-time conditioning adaptation reduces feature variance and improves generative stability under mild Lipschitz assumptions, while exhibiting a principled bias-variance tradeoff that governs the optimal strength of adaptation. Extensive experiments on state-of-the-art talking-head generators and benchmark datasets demonstrate consistent improvements in lip-sync accuracy, temporal coherence, identity preservation, and perceptual fidelity.
TT-SAC offers a model-agnostic and training-free strategy for enhancing generative video models, establishing test-time conditioning adaptation as an effective mechanism for stabilizing audio-driven portrait animation.
\end{abstract}


% \begin{figure}[tbp]
% \centering
% \subfloat[Static reference.]{%
% \includegraphics[width=0.33\linewidth]{figures/Picture1.pdf}
% \label{fig:baseline}}
% % \hfill
% \subfloat[Feature drift.]{%
% \includegraphics[width=0.33\linewidth]{figures/Picture2.pdf}
% \label{fig:tsne}}
% % \hfill
% \subfloat[TT-SAC feedback.]{%
% \includegraphics[width=0.33\linewidth]{figures/Picture3.pdf}
% \label{fig:scra}}

% \caption{Motivation for Self-Consistent Reference Aggregation (TT-SAC): mitigating identity drift.  
% (a) A static reference causes temporal artifacts and identity drift in audio-driven portrait generation.  
% (b) In feature space, the static reference (black cross) leads to a drifting trajectory (red), while TT-SAC \textcolor{blue}{(blue)} maintains compact, self-consistent identity features.  
% (c) The proposed TT-SAC introduces a lightweight feedback loop that stabilizes generation and preserves identity consistency over time.}
% \label{fig:moti}
% \end{figure}

\begin{figure}[tbp]
\centering
\subfloat[Identity inconsistency.]{%
\includegraphics[width=\linewidth]{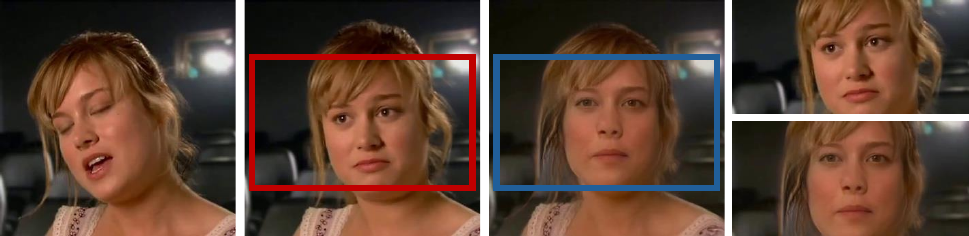}
}\\
\subfloat[Accumulated temporal drift.]{%
\includegraphics[width=\linewidth]{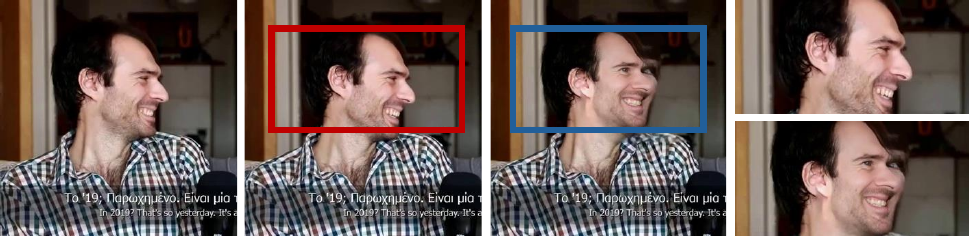}
}\\
\subfloat[Fine-grained visual artifacts.]{%
\includegraphics[width=\linewidth]{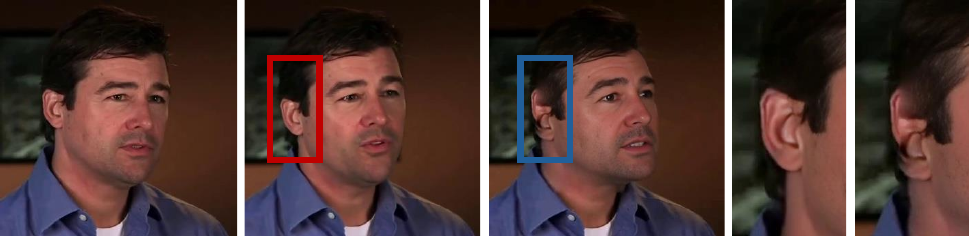}
}
\caption{Common failure cases in existing audio-driven talking-head generation. Columns from left to right show the reference frame, the real frame, the generated frame, and zoomed comparisons. Red and blue boxes mark regions of interest in the real and generated frames, respectively, which are magnified in Column 4. Existing models exhibit identity inconsistency, long-sequence drift, and local structural artifacts (\eg, ear fragmentation), where generated frames deviate from the reference more than real frames, an important unsolved problem to be addressed in this paper.
}
\label{fig:moti}
\end{figure}

\section{Introduction}

\IEEEPARstart{A}{udio-driven} talking-head generation~\cite{siarohin2019first, chen2019hierarchical, prajwal2020lip, xu2024hallo, cui2024hallo2, cui2025hallo3, cui2025hallo4, cheng2024dawn, wei2024aniportrait, guo2024liveportrait, xu2024vasa, zhang2026talkinghead} aims to synthesize photorealistic facial animations from a single portrait image and an input speech signal. Recent advances, including SadTalker~\cite{zhang2023sadtalker}, AniTalker~\cite{liu2024anitalker}, FLOAT~\cite{ki2025float}, and Sonic~\cite{ji2025sonic}, have significantly improved visual realism and lip synchronization by leveraging expressive motion representations and generative architectures such as diffusion and flow-based models.
Most of the existing methods follow the same inference paradigm: a \emph{single static reference image} is encoded once to produce an identity representation that conditions the entire video generation process while the pretrained generator remains fixed.

While effective, this paradigm has a fundamental limitation. A single portrait image provides only a partial observation of the underlying facial identity, typically captured under a fixed pose, expression, and illumination.
In contrast, the generated video involves rich and dynamically evolving facial motions driven by speech.
As generation progresses, the static conditioning representation may become increasingly misaligned with the evolving motion dynamics, often leading to identity drift, temporal inconsistency, or local facial artifacts.
Figure~\ref{fig:moti} illustrates representative examples where generated frames deviate from the target identity or exhibit abnormal facial distortions. These observations suggest that conditioning representations in generative video models should not remain static at inference stage.

Based on above observation, we revisit the conditioning mechanism of talking-head generation and ask the question: \emph{Can a pretrained generator refine its conditioning representation during inference without retraining or modifying model parameters?}
In this work, we propose \textit{Test-Time Self-Adaptive Conditioning (TT-SAC)}, a training-free inference framework that adapts the conditioning representations of pretrained talking-head generators using their own generated outputs. % The key idea is to compose the generator with its encoder during inference, forming a feedback loop that refines the identity representation through self-consistency. 
Instead of treating the reference portrait as immutable, TT-SAC first generates an initial video using the pretrained model, then re-encodes early generated frames to extract additional identity-aware features. These features are aggregated to construct a refined conditioning representation, which is subsequently used for a second generation pass. This generator-encoder composition effectively transforms a conventional one-pass generator into a self-adaptive system capable of correcting conditioning mismatch during inference.

Our approach follows a \emph{self-consistency principle}: generated frames provide additional observations of the target identity under diverse facial motions, revealing identity characteristics that may not be fully captured by the original portrait image.
% By incorporating this information through encoder-based feature aggregation, TT-SAC approximates a self-consistent equilibrium between identity representation and generated motion dynamics.
By incorporating identity cues extracted from these generated frames through encoder-based feature aggregation, TT-SAC approximates a self-consistent equilibrium between identity representation and generated motion dynamics.
Importantly, TT-SAC requires \emph{no gradient updates, retraining, or architectural modifications}.
It operates purely at inference stage and can be seamlessly applied to a wide range of pretrained talking-head generators.

\begin{figure*}[tbp]
\centering
\subfloat[Static identity conditioning.]{%
\includegraphics[width=0.23\textwidth]{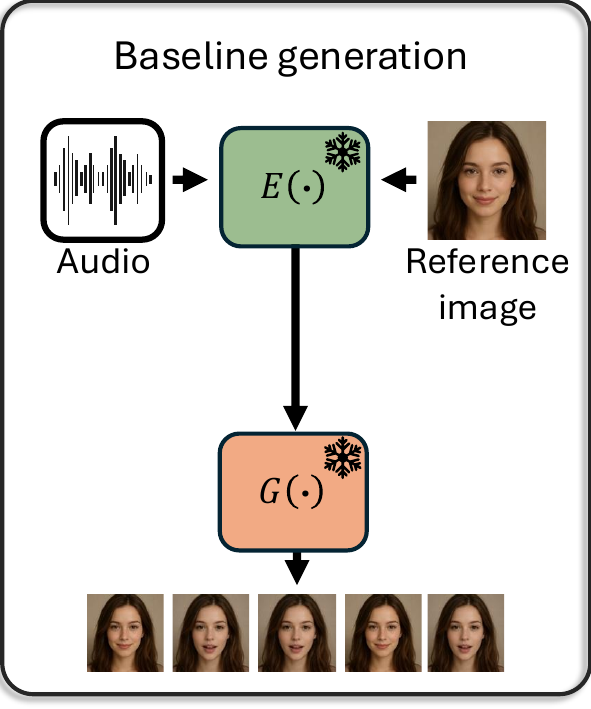}
}
\hfill
\subfloat[Test-time conditioning refinement.]{%
\includegraphics[width=0.28\textwidth]{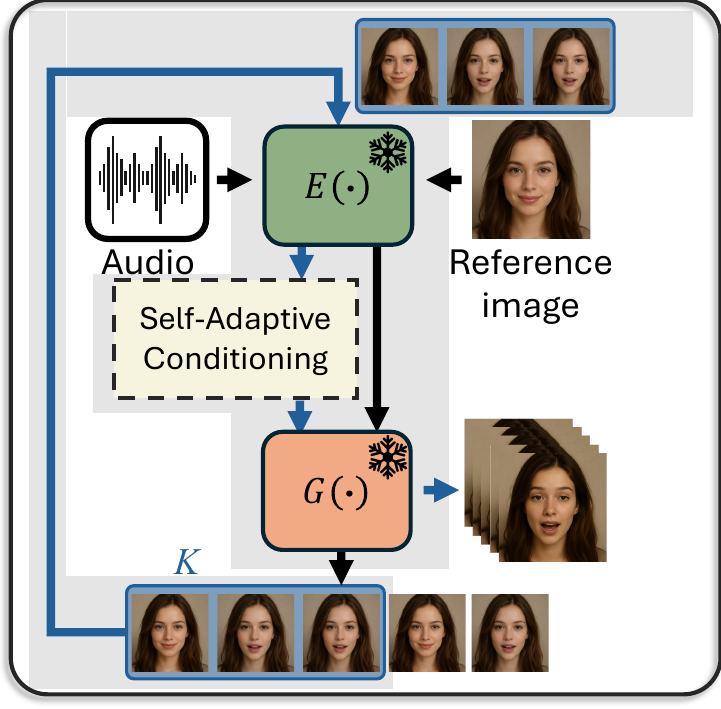}
}
\hfill
\subfloat[Self-consistent generator-encoder feedback.]{%
\includegraphics[width=0.40\textwidth]{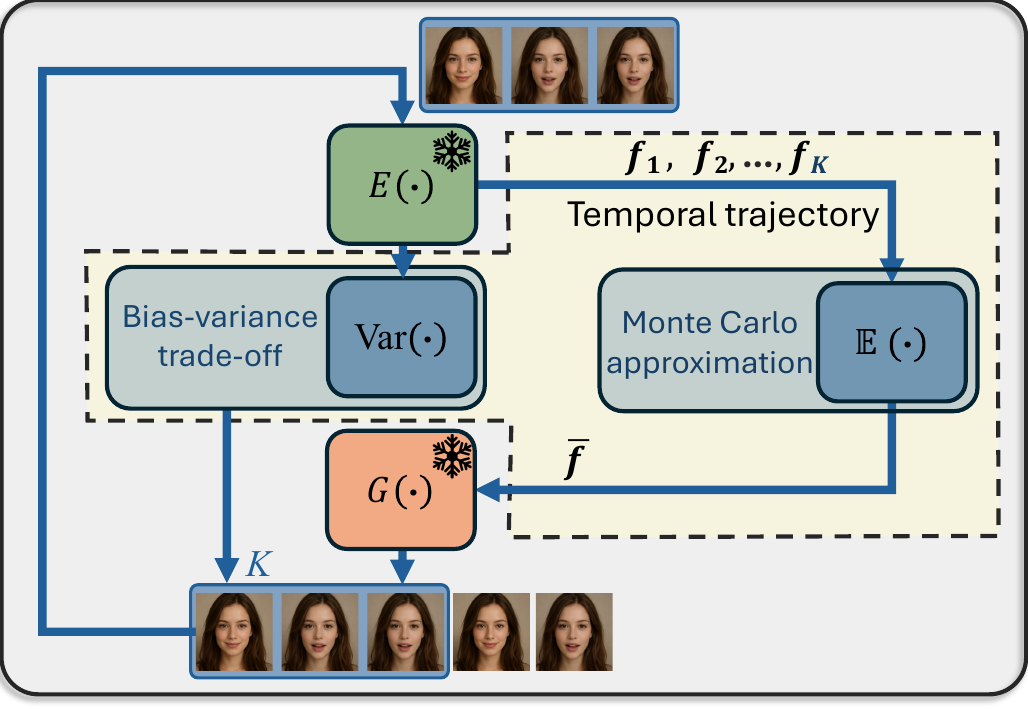}
}
\caption{Test-Time Self-Adaptive Conditioning (TT-SAC).
Conventional talking-head generators use a single identity embedding extracted from the reference image to condition the entire sequence. As facial motion and pose evolve, this static representation may become misaligned with generated frames, causing identity drift and artifacts. TT-SAC introduces a test-time mechanism that refines the conditioning feature without retraining. After an initial generation pass, several frames are re-encoded and their identity features are aggregated to produce an updated conditioning representation. This generator-encoder feedback approximates a self-consistent conditioning feature under the generative dynamics. The aggregation acts as a Monte Carlo estimate of a fixed point of the generator-encoder operator, reducing latent variance and improving identity stability.
}
\label{fig:scra_pipeline}
\end{figure*}

We provide theoretical insights into why conditioning adaptation improves generation stability. We show that the refinement process reduces feature variance and stabilizes generative dynamics under mild Lipschitz assumptions. Furthermore, the adaptation procedure can be interpreted as a stochastic fixed-point iteration, revealing a bias-variance tradeoff that characterizes the optimal level of conditioning refinement.
Extensive experiments on state-of-the-art pretrained generators, including FLOAT, Sonic, SadTalker, JoyVASA, and AniTalker, demonstrate that TT-SAC consistently improves lip synchronization accuracy, temporal coherence, identity preservation, and perceptual quality.
% The results demonstrate that test-time conditioning adaptation constitutes a simple yet powerful mechanism for stabilizing generative video models.
The results demonstrate that test-time conditioning adaptation provides a principled mechanism for improving the stability of generative video models.
Our contributions are summarized as follows:
\renewcommand{\labelenumi}{\roman{enumi}.}
\begin{enumerate}[leftmargin=0.4cm]
\item We introduce \textit{Test-Time Self-Adaptive Conditioning (TT-SAC)}, a new inference paradigm that enables pretrained talking-head generators to adapt their conditioning representations at test time without retraining, fine-tuning, or gradient updates.
\item A novel \textit{generator-encoder compositional inference framework} is proposed that refines identity conditioning through self-consistency between generated frames and conditioning features, transforming a static one-pass generator into a self-adaptive generation process.
\item We provide \textit{theoretical analysis} showing that conditioning adaptation reduces feature variance and stabilizes generative dynamics, and interpret the adaptation step as a stochastic fixed-point iteration with a principled bias-variance tradeoff.
\item Extensive experiments on state-of-the-art talking-head generators demonstrate consistent improvements, % in lip synchronization, temporal coherence, identity preservation, and perceptual quality, 
highlighting the generality and effectiveness of test-time conditioning adaptation.
\end{enumerate}

\section{Related Work}

Our work lies at the intersection of audio-driven talking-head generation and test-time adaptation for generative models. 
The related work is organised into three categories: (i) audio-driven talking-head generation, (ii) temporal consistency and motion modeling, and (iii) inference-time optimization in generative models. 
We clarify how our proposed TT-SAC differs fundamentally from each category of these research.

\textbf{Audio-driven talking-head generation.}
Audio-driven portrait animation has advanced rapidly in recent years. 
Early works focused on driving lip motion using facial landmarks or parametric 3D models~\cite{kumar2017obamanet, kim2018deep, suwajanakorn2017synthesizing, karras2017audio, chung2016out, jalalifar2018speech, yi2020audio}. 
More recent approaches directly synthesize photorealistic video from a single portrait and speech signal~\cite{thies2016face2face, wang2021one, ma2023dreamtalk, wang2025fantasytalking, kong2025let, ma2025playmate, chen2025echomimic}. 
For instance, SadTalker~\cite{zhang2023sadtalker} predicts 3D motion coefficients to animate stylized faces, while Sonic~\cite{ji2025sonic} disentangles intra- and inter-clip audio cues to improve long-range temporal coherence. 
FLOAT~\cite{ki2025float} employs flow-matching in a learned motion latent space for temporally consistent and efficient generation.

These methods primarily focus on architectural innovation, motion representation learning, and large-scale training. 
In contrast, our work does not modify model architectures or training strategies. 
% Instead, we revisit the \emph{test-time conditioning paradigm} itself.
% Instead, we introduce a test-time conditioning perspective.
We identify conditioning mismatch between static reference features and dynamic motion generation as an overlooked limitation and propose adapting the conditioning representation at test time without retraining.

\textbf{Temporal consistency and motion modeling.}
Ensuring temporal coherence, \eg, stable identity, smooth motion transitions, and consistent head pose, is a central challenge in portrait animation. 
Existing approaches address this through explicit temporal modeling, such as motion latent flows~\cite{drobyshev2022megaportraits}, long-range audio encoders~\cite{ji2025sonic}, disentangled motion representations~\cite{tan2025disentangle}, or multi-frame feature fusion~\cite{ding2025learnable, liu2024multimodal, guan2023stylesync}. 
These methods embed temporal structure directly into network design and require dedicated training procedures.

Our approach differs fundamentally: TT-SAC introduces no additional temporal modules and requires no retraining. 
Rather than explicitly modeling motion dynamics, we refine the conditioning representation using the generator’s own outputs. 
This implicit adaptation aligns identity features with the generated motion manifold and improves stability in a plug-and-play manner. 
Therefore, TT-SAC is complementary to existing temporal modeling architectures and can be applied to pretrained systems without architectural modification.

\textbf{Inference-time optimization in generative models.}
Inference-time optimization has attracted growing interest in diffusion and generative modeling, including self-conditioning, latent optimization, iterative refinement, and consistency-based sampling~\cite{ho2020denoising, song2020denoising, lipman2022flow, lu2022dpm}. 
These techniques typically focus on improving sample quality through modified sampling trajectories, additional optimization steps, or guidance mechanisms. 
However, they generally operate within the generative process itself and do not reconsider the conditioning representation.

In contrast, TT-SAC addresses a distinct problem: \emph{conditioning adaptation}. 
We introduce a feedback mechanism in which generated frames are re-encoded to construct an updated conditioning feature, enforcing a self-consistency principle between conditioning and generated outputs. 
This perspective differs from temporal smoothing, latent refinement, or gradient-based inference-time optimization, as it modifies neither network parameters nor sampling dynamics. 
Instead, it adapts the conditioning input through a single self-consistent update step.
To our knowledge, this is the first work in talking-head generation to explicitly formulate test-time conditioning adaptation as a self-consistency problem and demonstrate its effectiveness on state-of-the-art models.

% Our method complements architecture-centric advances of the first group, motion/dynamics modeling of the second group, and fills a gap in inference-time methods for audio-driven talking-head generation. 
% We provide theoretical motivation (variance reduction, fixed-point interpretation, bias-variance tradeoff) and practical plug-and-play utility, which distinguishes our work from prior art.

% Scalars are denoted by standard letters (\eg, $x$), vectors by bold lowercase letters (\eg, $\vx$), matrices by bold uppercase letters (\eg, $\mX$), and tensors by calligraphic symbols (\eg, $\gX$).

%for stabilizing pretrained audio-driven talking-head generators. 

\section{Method}

We present Test-Time Self-Adaptive Conditioning (TT-SAC; see Fig. \ref{fig:scra_pipeline} for an overview), a principled, test-time framework designed to stabilize pretrained audio-driven talking-head generators.
Rather than modifying network parameters or introducing new temporal modules, TT-SAC adapts the \emph{conditioning representation} at test time based on a stability criterion derived from the generative dynamics itself.

Unlike heuristic post-processing or smoothing strategies, TT-SAC originates from a formal analysis of conditioning instability in temporally evolving generative models. 
We first characterize this instability, then derive a fixed-point formulation for stable conditioning, and finally present a practical estimator that realizes this formulation without gradient updates or retraining. Below, we first introduce our notation.

\textbf{Generator-encoder composition.} Let $\gI_r$ denote a reference portrait image and $\mA \!=\! \{\va_t\}_{t=1}^{T}$ an input audio sequence of length $T$. 
Modern pretrained talking-head generators first encode the reference image into a latent identity representation
$\vf_r \!=\! E(\gI_r)$, where $E(\cdot)$ is an encoder that extracts the subject-specific identity features. 
The generator $G(\cdot)$ then produces a video sequence
$\{\hat{\gI}_t\}_{t=1}^{T} \!=\! G(\vf_r, \mA)$, 
with each frame $\hat{\gI}_t$ corresponding to audio input $\va_t$. 

Importantly, the latent feature $\vf_r$ remains fixed throughout inference. 
To formalize the interplay between generation and re-encoding, we define the \textit{generator-encoder composition}:
\begin{equation}
    (E \circ G)(\vf, \mA) = E(G(\vf, \mA)),
\end{equation}
which maps a conditioning feature $\vf$ and motion input $\mA$ to the latent identity feature obtained after one generation-encoding cycle. This composition serves as the core operator for our fixed-point analysis.
The \emph{static conditioning assumption} implicitly presumes that a single embedding $\vf_r$ can fully capture the subject's appearance across all poses, expressions, and dynamic motions. 
In practice, this assumption can lead to identity drift and temporal inconsistencies, especially when facial motions or expressions vary significantly. 

In Sec. \ref{sec:tt-sac} we describe the TT-SAC algorithm, Sec. \ref{sec:self_consistency} formalizes identity self-consistency, and Sec. \ref{sec:theory} provides theoretical insights on stability and variance reduction.

\subsection{Test-Time Self-Adaptive Conditioning}
\label{sec:tt-sac}

\textbf{Conditioning instability in dynamic generation.} 
Even with a fixed reference feature, the identity features of generated frames can evolve over time due to changes in pose, expression, and micro-motion. 
Let $\vf_t =  
(E \circ G)(\vf_r, \mA)_t$, $t = 1,\dots,T$,
denote the encoded features of the generated frames, forming a temporal trajectory $\{\vf_t\}_{t=1}^{T}$ in latent space. 
Ideally, a perfectly stable reference feature would satisfy $\vf_r = \frac{1}{T} \sum_{t=1}^{T} \vf_t$ or more generally, 
    $\vf_r = \mathbb{E}_t[\vf_t]$,
\ie, the mean feature of the generated sequence matches the reference. 
In practice, this rarely holds: feature drift accumulates along the generator-encoder composition, resulting in identity deviations and reduced temporal coherence.

This motivates a stability requirement: \emph{a conditioning feature should be consistent with the average features of the sequence it generates}. 
Formally, let $\vf \in \mathbb{R}^d$ denote a candidate identity feature. Given motion input $\mA$,  
a feature $\vf^*$ is \emph{stable} if it satisfies the \emph{self-consistency condition} $\vf^* = \frac{1}{T} \sum_{t=1}^{T} (E \circ G)(\vf^*, \mA)_t 
$ or equivalently, 
$\vf^* = \mathbb{E}_t [(E \circ G)(\vf^*, \mA)_t]$. 
In other words, when conditioning on $\vf^*$, the generated frames, after being encoded back into feature space, produce the same average feature representation $\vf^*$. 
This defines a \emph{self-consistency fixed point} of $E \circ G$,  
providing a principled target for stabilizing identity under dynamic motion.

\textbf{Fixed-point conditioning formulation.}
We define the conditioning operator 
$\mathcal{T}: \mathbb{R}^d \rightarrow \mathbb{R}^d$, 
\begin{equation}
    \mathcal{T}(\vf) = \mathbb{E}_t\big[(E \circ G)(\vf, \mA)_t\big].
    \label{eq:tf}
\end{equation}
A \emph{stable conditioning feature} $\vf^*$ is a fixed point of this operator:
$\vf^* = \mathcal{T}(\vf^*)$.
TT-SAC can thus be interpreted as a fixed-point iteration on the generator-encoder composition, aiming to find $\vf^*$ that aligns the conditioning feature with the statistics of the generated sequence.
As shown in Proposition~\ref{prop:scc} (Sec.~\ref{sec:self_consistency}), the conditioning feature satisfies self-consistency.

Rather than assuming that the original reference embedding $\vf_r$ is already stable, TT-SAC approximates a feature satisfying $\vf^* \approx \mathcal{T}(\vf^*)$, 
thereby producing a self-consistent identity embedding under the generator-encoder dynamics.

\textbf{Practical estimation via Monte Carlo approximation.} 
The conditioning operator (Eq. \ref{eq:tf}) involves an expectation over the generated sequence, which is generally intractable to compute exactly. 
We therefore approximate it using a finite-sample Monte Carlo estimator.

Specifically, given the reference feature $\vf_r$, we evaluate the generator-encoder composition on $K$ generated frames and form the empirical estimate:
\begin{equation}
    \widehat{\mathcal{T}}(\vf_r)
    = \frac{1}{K} \sum_{t=1}^{K} (E \circ G)(\vf_r, \mA)_t.
\end{equation}
This yields an aggregated feature:
\begin{equation}
    \bar{\vf} = \widehat{\mathcal{T}}(\vf_r),
\end{equation}
which approximates $\mathcal{T}(\vf_r)$, with variance decreasing as $K$ increases. 
Each encoded frame is a sample from the generator-encoder dynamics; averaging them reduces motion-induced fluctuations in the latent space.
% Each term $(E \circ G)(\vf_r, \mA)_t$ can be interpreted as a sample produced by the dynamical system induced by the generator-encoder composition, and temporal averaging over these encoded features mitigates motion-induced fluctuations in the latent space. 

Thus, $\bar{\vf}$ provides a stochastic estimate of a self-consistent conditioning feature. 
We then perform a single refinement step:
\begin{equation}
    \vf_r \leftarrow \bar{\vf},
\end{equation}
corresponding to one empirical fixed-point iteration 
$\vf_r \mapsto \widehat{\mathcal{T}}(\vf_r)$. 
This update moves conditioning feature toward a stable solution without retraining or gradient-based optimization.

\textbf{A bias-variance trade-off.}  
The number of sampled frames $K$ determines the statistical accuracy of the Monte Carlo estimator $\widehat{\mathcal{T}}(\vf_r)$. From standard Monte Carlo theory \cite{metropolis1949monte}, averaging multiple samples reduces the variance of the empirical mean. In the ideal case of independent features, the variance scales inversely with the number of samples:
\begin{equation}
\mathrm{Var}(\bar{\vf}) = \frac{1}{K} \mathrm{Var}\big((E \circ G)(\vf_r,\mA)_t\big).
\end{equation}
Increasing $K$ therefore suppresses estimation noise, yielding a more stable empirical estimate of $\mathcal{T}(\vf_r)$ and, by Lemma~\ref{lemma:var}, more stable generator outputs.

In practice, the generated feature sequence is generally non-stationary: $\vf_t = (E \circ G)(\vf_r,\mA)_t$ may drift over time due to evolving pose, expression, and micro-motion. Averaging over too many frames can thus introduce \emph{motion-induced bias}, as later features systematically deviate from the initial identity statistics.  
This phenomenon is formalized by the bias-variance decomposition in Proposition \ref{prop:bias-var} (Sec. \ref{sec:theory}), where the expected squared deviation of $G(\bar{\vf},\mA)$ % from $G(\vmu,\mA)$ 
separates into two components:  
\begin{enumerate} % : $\|J_G(\vmu)(\vmu_K-\vmu)\|^2$,
    \item \textit{Bias} captures systematic drift in the aggregated feature due to temporal non-stationarity.  % $\mathrm{tr}\big(J_G(\vmu)\, \mathrm{Cov}(\bar{\vf})\, J_G(\vmu)^\top\big)$, 
    \item \textit{Variance} captures random fluctuations due to finite sampling.
\end{enumerate}

Consequently, the number of sampled frames $K$ directly governs a classical bias-variance trade-off: increasing $K$ reduces variance but amplifies bias due to motion-induced drift. The optimal number of frames can be expressed abstractly as
\begin{equation}
K^* = \arg\min_K 
\left(
\frac{\sigma^2}{K} + \mathrm{Bias}^2(K)
\right),
\label{eq:k}
\end{equation}
where $\sigma^2$ is the per-frame feature variance and $\mathrm{Bias}(K)$ quantifies the systematic deviation introduced by temporal evolution over the first $K$ frames.  
Small values of $K$ are sufficient to reduce variance while limiting motion-induced bias, consistent with this theoretical perspective. 

This analysis formalizes the intuitive trade-off discussed in Sec. \ref{sec:self_consistency}, showing that Monte Carlo averaging in TT-SAC approximates a stable fixed-point feature while controlling both statistical fluctuations and temporal drift.

\textbf{Why conditioning-level adaptation?} 
TT-SAC operates directly in the conditioning space, modifying the identity embedding rather than the generator parameters or output frames. 
This choice is principled: the instability arises from a mismatch between the fixed conditioning feature and the identity statistics induced by the generator-encoder composition. 
Adapting the conditioning feature therefore addresses the root cause of drift. 
Compared to gradient-based inference-time optimization, TT-SAC avoids backpropagation through multi-step generative processes, eliminating additional memory overhead, optimization instability, and the risk of altering pretrained identity priors. 
Unlike pixel-level temporal smoothing, which post-processes outputs, TT-SAC influences the generator trajectory itself by updating the latent conditioning input. 

The framework is model-agnostic and can apply broadly to diffusion-based, flow-based, or implicit keypoint architectures, provided the conditioning pathway and identity encoder are accessible. By operating at inference stage in conditioning space, TT-SAC implements a gradient-free, self-consistent adaptation mechanism that preserves pretrained weights while improving temporal consistency and identity fidelity. It can be applied not only to the identity (appearance) stream but also to motion or other driving signals when available, stabilizing multi-stream outputs without retraining or architectural modifications.

We next provide a formal analysis explaining why this simple averaging step reduces variance, improves stability, and approximates a fixed-point iteration.

\subsection{Identity Self-Consistency}
\label{sec:self_consistency}

The fixed-point formulation of TT-SAC can be interpreted as enforcing \emph{identity self-consistency} between the conditioning feature and the generated frames.

\begin{assumption}[Identity feature consistency]
Let $E(\cdot)$ denote the identity encoder.  
For images of the same subject, the encoded features concentrate 
around a subject-specific mean representation $\vmu$, \ie,
$\vf = E(\gI) = \vmu + \boldsymbol{\epsilon}$,
where $\mathbb{E}[\boldsymbol{\epsilon}] = 0$ and 
$\mathrm{Cov}(\boldsymbol{\epsilon})$ is bounded.
\end{assumption}

Under this assumption, the encoded features of generated frames
$\vf_t$ ($t \!=\! 1,\dots,T$), 
can be viewed as noisy samples of the identity representation 
induced by the conditioning feature $\vf$ under motion sequence $\mA$.
We define an \emph{identity self-consistency objective}:
% \begin{equation}
$L(\vf) \!=\! \mathbb{E}_t \big[ \| E(G(\vf, \mA)_t) \!-\! \vf \|^2 \big]$,
% \end{equation}
which quantifies the discrepancy between the conditioning feature 
and the identity features observed in the generated sequence.

\begin{proposition}[Self-consistent conditioning]
Any feature $\vf^*$ satisfying the fixed-point condition
$\vf^* = \mathcal{T}(\vf^*) = \mathbb{E}_t \big[(E \circ G)(\vf^*, \mA)_t\big]$
is a stationary point of the identity self-consistency objective $L(\vf)$, 
up to first-order approximation.
\label{prop:scc}
\end{proposition}

\begin{proof}
The gradient of $L(\vf)$ with respect to $\vf$ is
% The gradient of the identity self-consistency objective with respect to $\vf$ is
\[
\nabla_\vf L(\vf) 
= 2\mathbb{E}_t \big[ \vf - E(G(\vf, \mA)_t) \big] 
- 2\mathbb{E}_t \Big[ J_\vf^\top \big(E(G(\vf, \mA)_t) - \vf\big) \Big],
\]
where $J_\vf = \frac{\partial E(G(\vf, \mA)_t)}{\partial \vf}$ is the Jacobian of the generated feature with respect to $\vf$.

For fixed-point analysis, it is standard to adopt a first-order approximation by neglecting the Jacobian term, which yields
\[
\nabla_\vf L(\vf) \approx 2  \big(\vf - \mathbb{E}_t[E(G(\vf, \mA)_t)]\big).
\]

Setting $\nabla_\vf L(\vf) = 0$ immediately gives
\[
\vf = \mathbb{E}_t[E(G(\vf, \mA)_t)] = \mathcal{T}(\vf),
\]
demonstrating that any fixed point of $\mathcal{T}$ is a stationary point of $L(\vf)$ under this approximation.
\end{proof}

This interpretation formalizes TT-SAC as seeking a conditioning feature that is statistically consistent with the identity features produced by the generator under motion (\ie, a fixed point of the operator $\mathcal{T}$ in Eq. \ref{eq:tf}), thereby reducing identity drift and improving temporal coherence.

\subsection{Theoretical Insights}
\label{sec:theory}

% Having motivated the self-consistency objective and practical Monte Carlo update, we now
Here we analyze the statistical properties of aggregated features and the fixed-point iteration to justify TT-SAC theoretically.

\begin{definition}[Conditioning operator]
Let $G(\cdot,\mA)$ be a pretrained generator and $E(\cdot)$ its identity encoder. 
Define the \emph{conditioning operator} $\mathcal{T} : \mathbb{R}^d \to \mathbb{R}^d$ as $\mathcal{T}(\vf) = \mathbb{E}_t \big[ (E \circ G)(\vf, \mA)_t \big]$,
where the expectation is over the generated frames. 
A feature $\vf^*$ is \emph{stable} if it satisfies $\vf^* = \mathcal{T}(\vf^*)$,
\ie, it is a fixed point of the operator. 
This formalizes the self-consistency principle that TT-SAC approximates at test time.
\end{definition}

\begin{proposition}[Covariance of aggregated features]
Let $\vf_t = E(G(\vf_r, \mA)_t)$ be the encoded feature of the $t$-th generated frame, and assume
$\vf_t = \vmu + \boldsymbol{\varepsilon}_t$, $\mathbb{E}[\boldsymbol{\varepsilon}_t] = 0$.
Define the aggregated feature
$\bar{\vf} = \frac{1}{K} \sum_{t=1}^K \vf_t$,
and the lagged covariance
$\boldsymbol{\Gamma}_\tau = \mathrm{Cov}(\boldsymbol{\varepsilon}_t, \boldsymbol{\varepsilon}_{t+\tau})$.
Then the covariance of the aggregated feature is
$\mathrm{Cov}(\bar{\vf}) = \frac{1}{K} \Big( \boldsymbol{\Gamma}_0 + 2 \sum_{\tau=1}^{K-1} \big(1 - \frac{\tau}{K}\big) \boldsymbol{\Gamma}_\tau \Big)$.
\end{proposition}

\begin{proof}
By definition, the covariance of the empirical mean is
\[
\mathrm{Cov}(\bar{\vf}) 
= \mathrm{Cov}\Big(\frac{1}{K}\sum_{t=1}^K \vf_t\Big) 
= \frac{1}{K^2} \sum_{i=1}^{K} \sum_{j=1}^{K} \mathrm{Cov}(\vf_i, \vf_j).
\]
Substituting $\vf_t \!=\! \vmu + \boldsymbol{\varepsilon}_t$ and noting that adding a constant does not affect covariance, we have $\mathrm{Cov}(\vf_i, \vf_j) \!=\! \mathrm{Cov}(\boldsymbol{\varepsilon}_i, \boldsymbol{\varepsilon}_j) \!=\! \boldsymbol{\Gamma}_{|i-j|}$.
Hence, $\mathrm{Cov}(\bar{\vf}) \!=\! \frac{1}{K^2} \sum_{i=1}^{K} \sum_{j=1}^{K} \boldsymbol{\Gamma}_{|i-j|}$.

We now reorganize the double sum according to the lag $\tau = |i-j|$. The term with $\tau=0$ appears $K$ times (when $i=j$), giving a contribution of $K \boldsymbol{\Gamma}_0$. For each lag $\tau = 1, \dots, K-1$, there are $K-\tau$ pairs $(i,j)$ with $|i-j| = \tau$, and by symmetry each contributes twice. Therefore, the total contribution for lag $\tau$ is $2 (K-\tau) \boldsymbol{\Gamma}_\tau$.
Dividing by $K^2$, we obtain
\begin{align}
    \mathrm{Cov}(\bar{\vf}) 
& = \frac{1}{K^2} \Big( K \boldsymbol{\Gamma}_0 + 2 \sum_{\tau=1}^{K-1} (K-\tau) \boldsymbol{\Gamma}_\tau \Big) \nonumber\\
& = \frac{1}{K} \Big( \boldsymbol{\Gamma}_0 + 2 \sum_{\tau=1}^{K-1} \big(1 - \frac{\tau}{K}\big) \boldsymbol{\Gamma}_\tau \Big), \nonumber
\end{align}
as claimed.
\end{proof}

This analysis follows standard time-series variance estimation, where the covariance of the averaged features depends on temporal correlations between frames. In practice, consecutive generated frames exhibit moderate temporal dependence due to smooth facial motion, which is captured by the lagged covariance terms $\boldsymbol{\Gamma}_\tau$.

\begin{lemma}[Output variance bound]
Assume the generator $G(\vf,\mA)$ is locally Lipschitz with respect to $\vf$:
$\|G(\vf_1,\mA)-G(\vf_2,\mA)\| \le L_G \|\vf_1-\vf_2\|$.
Then, for the aggregated feature $\bar{\vf}$ and mean $\vmu = \mathbb{E}[\vf_t]$,
$\mathbb{E}\big[\|G(\bar{\vf},\mA)-G(\vmu,\mA)\|^2\big] \le L_G^2  \mathrm{tr}(\mathrm{Cov}(\bar{\vf}))$.
\label{lemma:var}
\end{lemma}

\begin{proof}
By the Lipschitz property of $G$,
\[
\|G(\bar{\vf},\mA) - G(\vmu,\mA)\| \le L_G \|\bar{\vf}-\vmu\|.
\]
Squaring both sides gives
\[
\|G(\bar{\vf},\mA) - G(\vmu,\mA)\|^2 \le L_G^2 \|\bar{\vf}-\vmu\|^2.
\]
Taking expectation over the randomness in $\bar{\vf}$ yields
\[
\mathbb{E}\big[\|G(\bar{\vf},\mA)-G(\vmu,\mA)\|^2\big] \le L_G^2  \mathbb{E}\|\bar{\vf}-\vmu\|^2.
\]
Finally, by the standard property of covariance,
\[
\mathbb{E}\|\bar{\vf}-\vmu\|^2 = \mathrm{tr}(\mathrm{Cov}(\bar{\vf})),
\]
which completes the proof.
\end{proof}

\begin{lemma}[Approximate fixed-point iteration]
Let $\mathcal{T} : \mathbb{R}^d \to \mathbb{R}^d$ be locally contractive with constant $c < 1$:
$\|\mathcal{T}(\vf_1)-\mathcal{T}(\vf_2)\| \le c \|\vf_1-\vf_2\| \quad \forall \vf_1, \vf_2 \in \mathbb{R}^d$.
Then the exact fixed-point iteration $\vf^{(k+1)} = \mathcal{T}(\vf^{(k)})$ converges linearly to the unique fixed point $\vf^*$.  
Furthermore, the empirical Monte Carlo iteration used in TT-SAC,
$\vf^{(1)} = \hat{\mathcal{T}}(\vf^{(0)}) = \frac{1}{K} \sum_{t=1}^{K} (E \circ G)(\vf^{(0)},\mA)_t$,
satisfies $\mathbb{E}[\vf^{(1)}] = \mathcal{T}(\vf^{(0)})$, so a single step moves $\vf^{(0)}$ toward $\vf^*$ in expectation. 
\label{lemma:fix-point}
\end{lemma}

\begin{proof}
By Banach's fixed-point theorem \cite{banach1922operations}, a contractive operator $\mathcal{T}$ admits a unique fixed point $\vf^*$ such that $\mathcal{T}(\vf^*) = \vf^*$. For the exact iteration $\vf^{(k+1)} = \mathcal{T}(\vf^{(k)})$, we have
\[
\|\vf^{(k+1)} - \vf^*\| = \|\mathcal{T}(\vf^{(k)}) - \mathcal{T}(\vf^*)\| \le c \|\vf^{(k)} - \vf^*\|,
\]
which implies linear convergence:
\[
\|\vf^{(k)} - \vf^*\| \le c^k \|\vf^{(0)} - \vf^*\|.
\]

For the Monte Carlo estimator $\hat{\mathcal{T}}(\vf^{(0)})$, we have $\mathbb{E}[\hat{\mathcal{T}}(\vf^{(0)})] = \mathcal{T}(\vf^{(0)})$, so the expected update
\[
\mathbb{E}[\vf^{(1)} - \vf^*] = \mathcal{T}(\vf^{(0)}) - \vf^*,
\]
contracts toward $\vf^*$ with factor $c$. Hence, a single Monte Carlo iteration, as used in TT-SAC, moves the conditioning feature toward the fixed point on average. Empirically, additional iterations provide diminishing improvements and may introduce bias due to temporal averaging over the motion sequence.
\end{proof}

\begin{proposition}[Bias-variance decomposition]
Let $\bar{\vf} = \frac{1}{K} \sum_{t=1}^K \vf_t$ be the empirical mean of $K$ generated latent features, and let $\vmu = \mathbb{E}[\vf_t]$ denote the true mean feature, with $\vmu_K = \mathbb{E}[\bar{\vf}]$ the expected value of the averaged features.  
Linearizing the generator $G$ around $\vmu$ using its Jacobian $J_G(\vmu)$, the expected squared deviation of the generator output due to averaging can be decomposed as
\begin{align}
\mathbb{E}\big[\|G(\bar{\vf},\mA) - G(\vmu,\mA)\|^2\big] 
&\approx \underbrace{\|J_G(\vmu)(\vmu_K-\vmu)\|^2}_{\text{bias}} \nonumber \\
&\quad + \underbrace{\mathrm{tr}\Big(J_G(\vmu)  \mathrm{Cov}(\bar{\vf})  J_G(\vmu)^\top \Big)}_{\text{variance}}. \nonumber
\end{align}
\label{prop:bias-var}
\end{proposition}

This decomposition separates the contribution of systematic bias (due to temporal drift or non-stationarity in $\bar{\vf}$) from the variance (due to random fluctuations in the sampled features).

\begin{proof}
We start by applying a first-order Taylor expansion of $G$ around the true mean $\vmu$:
\[
G(\bar{\vf},\mA) \approx G(\vmu,\mA) + J_G(\vmu)(\bar{\vf}-\vmu),
\]
where $J_G(\vmu) = \frac{\partial G}{\partial \vf} \big|_{\vf = \vmu}$ is the Jacobian. 

Taking the squared norm and expectation yields
\begin{align}
\mathbb{E}\|G(\bar{\vf},\mA) - G(\vmu,\mA)\|^2 
&\approx \mathbb{E}\|J_G(\vmu)(\bar{\vf}-\vmu)\|^2 \nonumber \\
&= \underbrace{\|J_G(\vmu)(\vmu_K-\vmu)\|^2}_{\text{bias}} \nonumber \\ 
& + \underbrace{\mathrm{tr}\Big(J_G(\vmu)  \mathrm{Cov}(\bar{\vf})  J_G(\vmu)^\top \Big)}_{\text{variance}}, \nonumber
\end{align}
where we used the standard property that for a random vector $\vx$, $\mathbb{E}\|\vx\|^2 = \|\mathbb{E}[\vx]\|^2 + \mathrm{tr}(\mathrm{Cov}(\vx))$.  

Thus, the expected squared deviation of $G(\bar{\vf},\mA)$ decomposes naturally into a bias term capturing systematic shifts and a variance term capturing stochastic fluctuations of the averaged features.
\end{proof}

This decomposition justifies the frame-averaging trade-off in TT-SAC: larger $K$ reduces variance but can increase bias due to temporal drift, as formalized by the optimal frame number $K^*$ (Eq.~\ref{eq:k}). Building on this theoretical foundation, we now evaluate TT-SAC empirically. % We investigate how the frame-averaging trade-off impacts identity preservation, temporal consistency, and overall perceptual quality across diverse motion sequences and datasets. 

\begin{figure*}[tbp]
\centering

% ===== 第一行 =====
\subfloat[Eye deformation and lighting artifacts.]{%
\includegraphics[width=0.48\textwidth]{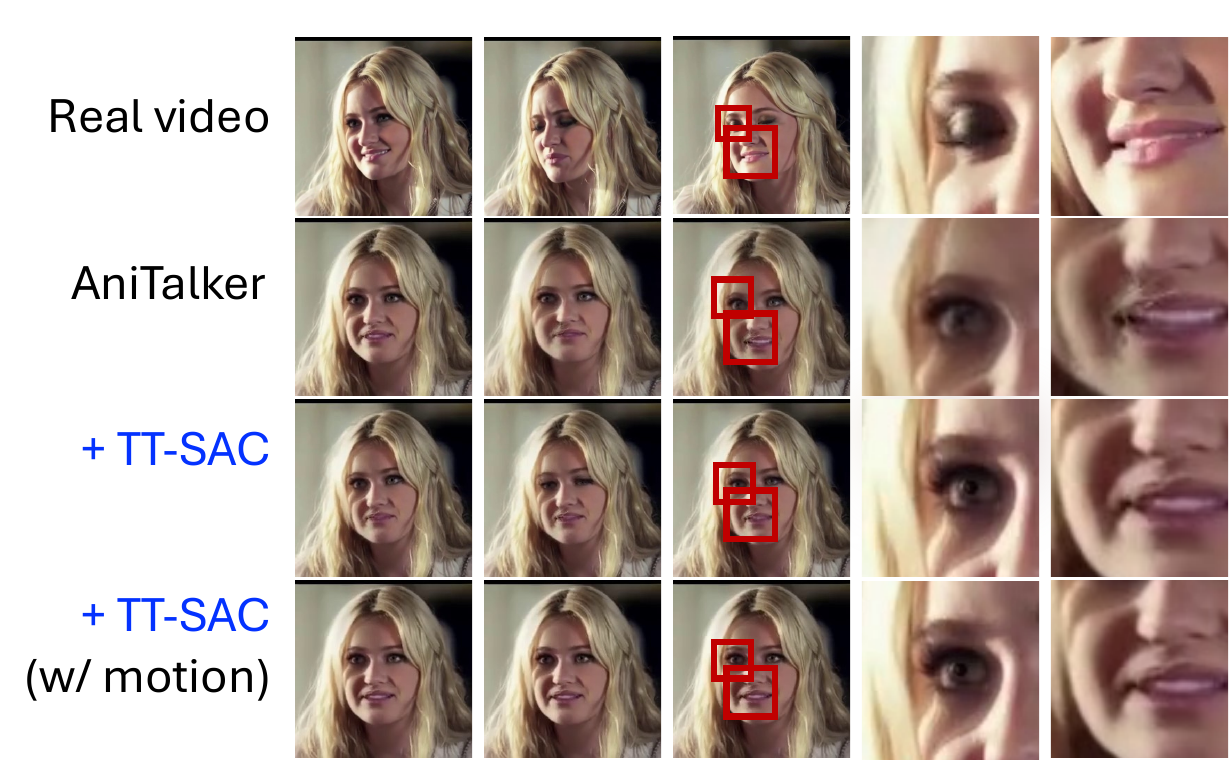}
}
\hfill
\subfloat[Sudden facial widening and identity drift.]{%
\includegraphics[width=0.48\textwidth]{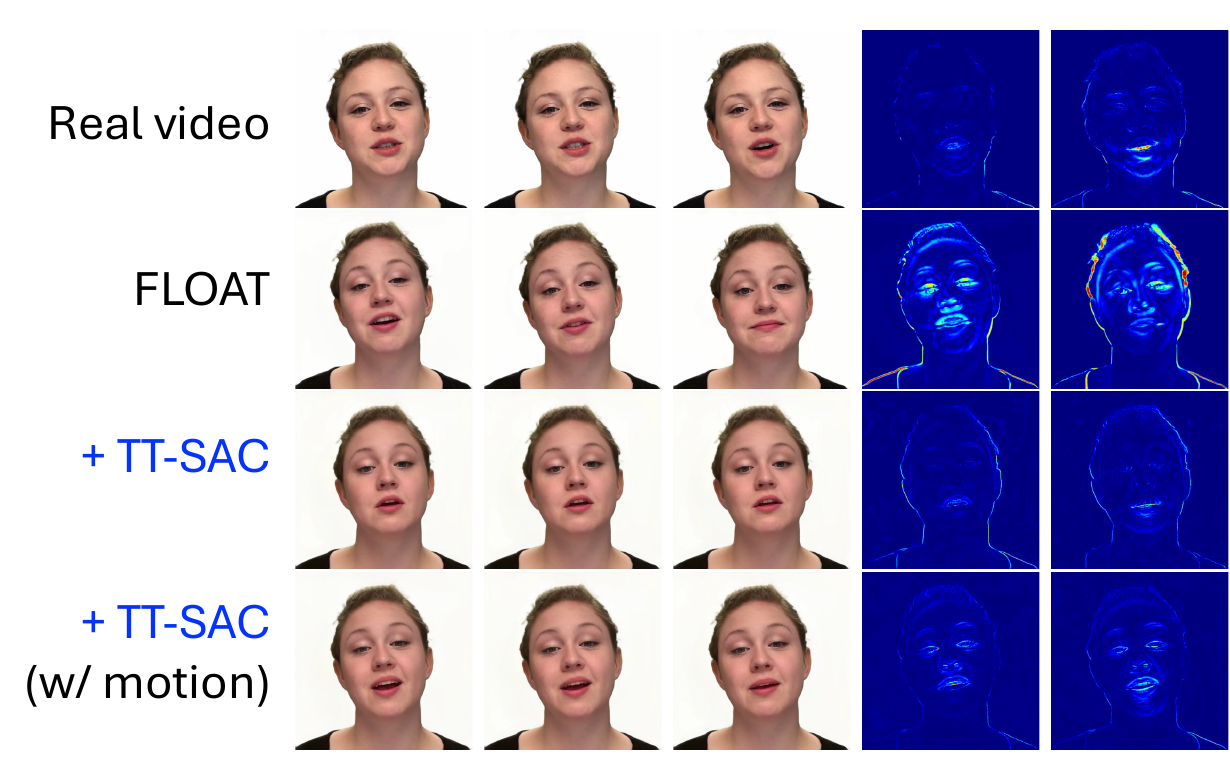}
}\\
% \hfill
\subfloat[Abrupt eye enlargement and blurred teeth.]{%
\includegraphics[width=0.48\textwidth]{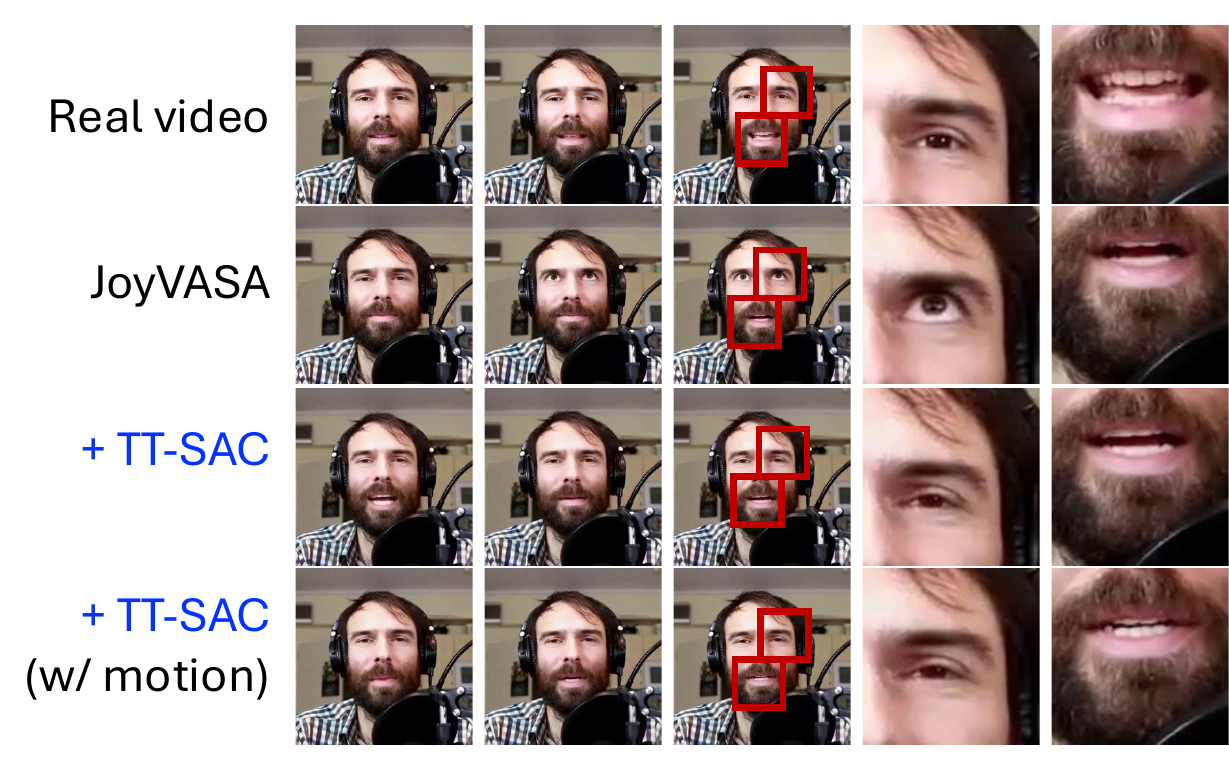}
}
\hfill
%\vspace{0.4cm}
% ===== 第二行 =====
\subfloat[Facial distortion and unnatural rotation.]{%
\includegraphics[width=0.48\textwidth]{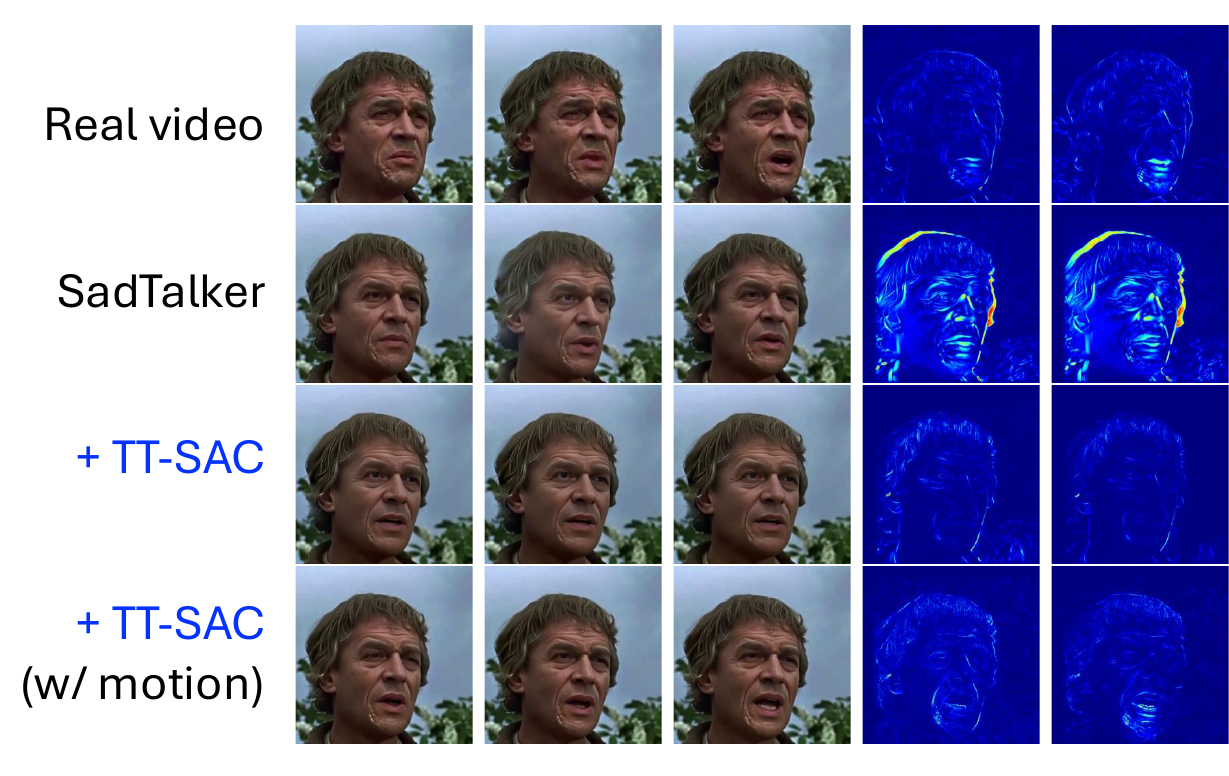}
}\\
% \hfill
\subfloat[Hand blur and audio-visual misalignment.]{%
\includegraphics[width=0.48\textwidth]{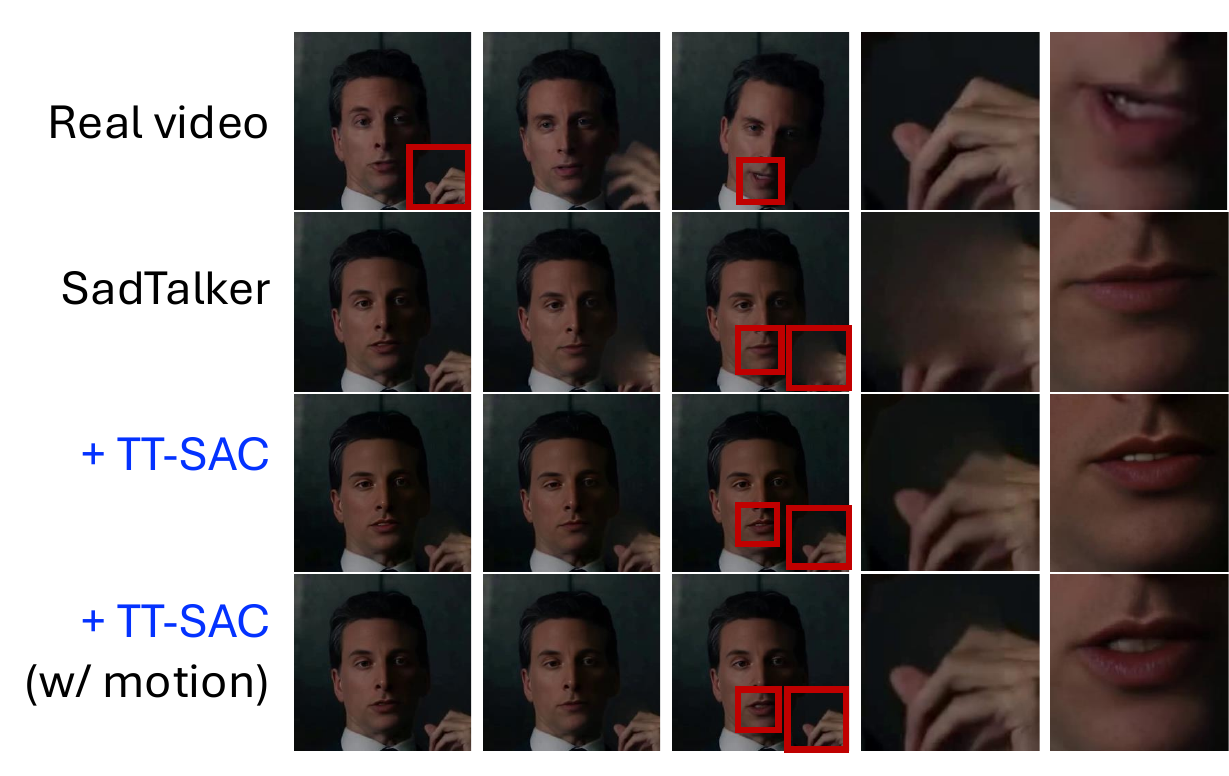}
}
\hfill
\subfloat[Sudden global facial shift and jitter.]{%
\includegraphics[width=0.48\textwidth]{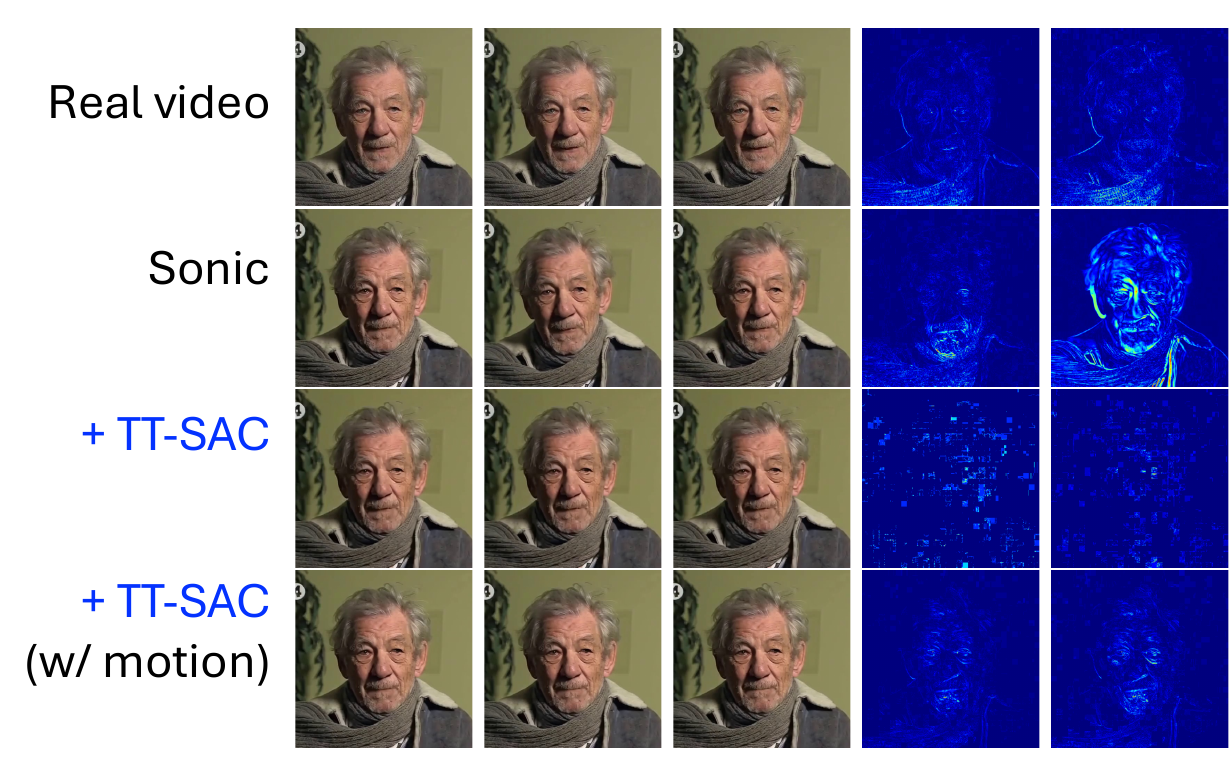}
}
\caption{
Qualitative comparison of typical failure cases. % in audio-driven talking-head generation.
Each block compares the real video against the baseline, baseline+TT-SAC, and TT-SAC on both appearance and motion streams.
% Each block compares the real video with three variants: the baseline model, the baseline augmented with our TT-SAC module, and TT-SAC applied to both the appearance and motion streams (TT-SAC with motion).
% Each block compares the real video with that generated by a baseline model, by the baseline enhanced with our TT-SAC module, and TT-SAC applied to both appearance and motion streams (TT-SAC w/ motion). 
The last columns show either zoomed-in regions (a, c, e) or inter-frame difference heatmaps (b, d, f), where brighter responses indicate larger temporal inconsistencies.
Baseline models exhibit various artifacts, including (a) eye deformation and lighting artifacts, (b) facial widening and identity drift, (c) eye enlargement and blurred teeth, (d) late-frame facial distortion and unnatural rotation, (e) hand blur and audio-visual misalignment, and (f) sudden global facial shifts causing jitter. 
By integrating TT-SAC, these artifacts are substantially reduced, and further improvements are observed when the module is also applied to the motion stream, yielding more stable facial structure, improved temporal consistency, and better identity preservation.
}
\label{fig:main}
\end{figure*}

\section{Experiment}

\label{sec:experiments}

% We present the experimental setup and subsequent analysis.

We first describe the experimental setup, followed by our evaluations and analysis.

\subsection{Setup}

\textbf{Datasets.} We evaluate TT-SAC on three benchmark datasets: Hallo~\cite{cui2025hallo3}, RAVDESS~\cite{livingstone2018ryerson}, and CelebV-HQ~\cite{zhu2022celebv}. All videos are standardized to 25 frames per second with audio resampled to 16 kHz, and facial regions are cropped and resized to $512\! \times\! 512$ pixels following standard preprocessing protocols. For evaluation, we select diverse clips to ensure robust analysis: 100 clips of 4 seconds from Hallo, 50 clips of 4-10 seconds from RAVDESS containing 24 distinct identities, and 50 clips of 4 seconds from CelebV-HQ. This selection captures variability in both identity and speaking style, supporting comprehensive quantitative and qualitative assessment.

\textbf{Models.} We compare TT-SAC against five state-of-the-art audio-driven talking-head generation methods with publicly available implementations: SadTalker~\cite{zhang2023sadtalker}, Sonic~\cite{ji2025sonic}, AniTalker~\cite{liu2024anitalker}, FLOAT~\cite{ki2025float}, and JoyVASA~\cite{cao2024joyvasa}. Experiments are conducted on the three benchmark datasets using all pretrained generators. We evaluate two variants of TT-SAC: (i) + TT-SAC, which applies fixed-point conditioning refinement only to the identity pathway, and (ii) + TT-SAC (w/ motion), which additionally refines motion pathways (\eg, keypoint streams or motion fields) when supported by the underlying model. Evaluation metrics cover lip-audio synchronization, temporal smoothness, perceptual quality, identity preservation, and overall video-level fidelity, providing a thorough comparison against baseline methods.

\textbf{Metrics.}
We assess the quality of generated videos along multiple dimensions. Fr\'echet Inception Distance (FID)~\cite{jayasumana2024rethinking} and 16-frame Fr\'echet Video Distance (FVD)~\cite{unterthiner2019fvd} evaluate overall visual and video-level fidelity. Audio-visual synchronization is measured using SyncNet scores (Sync-C and Sync-D)~\cite{chung2016out}. Identity preservation and perceptual similarity are quantified via cosine similarity of identity embeddings (CSIM)~\cite{Deng2019ArcFace} and Learned Perceptual Image Patch Similarity (LPIPS)~\cite{Zhang2018LPIPS}, respectively. Temporal coherence is measured using the Smooth score across consecutive frames.

\begin{table*}[tbp]
\centering
\small

\caption{Quantitative comparison on Hallo, CelebV-HQ, and RAVDESS.  
TT-SAC consistently improves lip-audio synchrony (Sync-C$\uparrow$ / Sync-D$\downarrow$), temporal smoothness (Smooth$\uparrow$), perceptual quality (LPIPS$\downarrow$), and identity preservation (CSIM$\uparrow$), leading to lower video-level distortion(FID$\downarrow$, FVD$\downarrow$) across all pretrained generators.  
\textcolor{blue}{+ TT-SAC} denotes our fixed-point conditioning refinement applied only to the identity pathway, while \textcolor{blue}{+ TT-SAC} (w/ motion) additionally applies the feedback to other driving signals (\eg, motion fields or keypoint streams) when supported by the underlying architecture.  
The results show that TT-SAC provides broad, model-agnostic improvements, and extending the refinement to motion pathways can yield further gains.
}
\resizebox{\textwidth}{!}{%
\begin{tabular}{crrrrrrrrrr
}
\toprule
\textbf{Dataset} & \textbf{Model} & \textbf{Venue} & \textbf{Method} &
\textbf{Sync-C$\uparrow$} & \textbf{Sync-D$\downarrow$} & \textbf{Smooth$\uparrow$} &
\textbf{LPIPS$\downarrow$} & \textbf{CSIM$\uparrow$} & \textbf{FID$\downarrow$} & \textbf{FVD$\downarrow$} \\
\midrule

% =====================
% Hallo Source Video
% =====================

\multirow{16}{*}{\textbf{Hallo}}
 & \multicolumn{3}{c}{\textit{Real video}}
 & 5.8875 & 8.2478 & 0.9945 & 0      & 1      & 0    & 0 \\
\addlinespace[0.1ex]  
\cline{2-11}
\addlinespace[0.1ex]  
% -------- AniTalker --------

& \multirow{3}{*}{AniTalker~\cite{liu2024anitalker}}& \multirow{3}{*}{ACMMM 2024}
  & Baseline
  & 3.9164 & 9.7782 & 0.9949 & 0.2762 & 0.7561 & 37.3635 & 143.3991 \\
& &  & \textcolor{blue}{+ TT-SAC}
  & \textbf{4.0822} & 9.8827 & 0.9951 & 0.2350 & 0.7990 & 27.4215 & 121.7679 \\
&  & & \textcolor{blue}{+ TT-SAC} (w/ motion)
  & 3.9488 & \textbf{9.6495} & \textbf{0.9954} & \textbf{0.1561} & \textbf{0.8445} & \textbf{22.1803} & \textbf{85.1056} \\
\addlinespace[0.1ex]
\cline{2-11}
\addlinespace[0.1ex]

% -------- FLOAT --------
& \multirow{3}{*}{FLOAT~\cite{ki2025float}}
  & \multirow{3}{*}{ICCV 2025} & Baseline
  & 3.4858 & 9.8774 & 0.9946 & 0.2423 & 0.7450 & 22.5672 & 129.1315 \\
& & & \textcolor{blue}{+ TT-SAC}
  & \textbf{3.5726} & \textbf{9.8724} & \textbf{0.9955} & 0.1810 & 0.7793 & \textbf{15.6304} & 109.6891 \\
& & & \textcolor{blue}{+ TT-SAC} (w/ motion)
  & 3.4446 & 9.9211 & 0.9952 & \textbf{0.1787} & \textbf{0.7995} & 15.7302 & \textbf{99.1639} \\
\addlinespace[0.1ex]  
\cline{2-11}
\addlinespace[0.1ex]  

% -------- JoyVASA --------
& \multirow{3}{*}{JoyVASA~\cite{cao2024joyvasa}}
  & \multirow{3}{*}{arXiv 2024} & Baseline
  & 6.4219 & 7.8281 & 0.9958 & 0.1311 & 0.8198 & 14.4476 & 119.5355 \\
& & & \textcolor{blue}{+ TT-SAC}
  & 5.4781 & 8.7902 & 0.9958 & \textbf{0.0720} & \textbf{0.8963} & \textbf{8.1284} & \textbf{68.5048} \\
& & & \textcolor{blue}{+ TT-SAC} (w/ motion)
  & \textbf{6.5690} & \textbf{7.7913} & \textbf{0.9959} & 0.0730 & 0.8882 & 9.0023 & 69.7138 \\
\addlinespace[0.1ex]  
\cline{2-11}
\addlinespace[0.1ex]  

% -------- Sadtalker --------
& \multirow{3}{*}{SadTalker~\cite{zhang2023sadtalker}}
  & \multirow{3}{*}{CVPR 2023} & Baseline
  & 5.4247 & 8.6527 & \textbf{0.9959} & 0.1424 & 0.7643 & 25.5395 & 127.6444 \\
&  & & \textcolor{blue}{+ TT-SAC}
  & 5.3596 & \textbf{8.6202} & 0.9955 & 0.0923 & 0.8247 & \textbf{18.1258} & \textbf{83.1818} \\
&  & & \textcolor{blue}{+ TT-SAC} (w/ motion)
  & \textbf{5.5373} & 8.6218 & 0.9955 & \textbf{0.0915} & \textbf{0.8255} & 21.1149 & 95.5873 \\
\addlinespace[0.1ex]  
\cline{2-11}
\addlinespace[0.1ex]  

% -------- Sonic --------
& \multirow{3}{*}{Sonic~\cite{ji2025sonic}}
  & \multirow{3}{*}{CVPR 2025} & Baseline
  & 6.4219 & 7.8281 & \textbf{0.9963} & 0.1552 & 0.8041 & 13.6096 & 92.4699 \\
&  & & \textcolor{blue}{+ TT-SAC}
  & 6.2633 & 7.9277 & \textbf{0.9963} & \textbf{0.1240} & \textbf{0.8465} & \textbf{12.3094} & \textbf{76.5809} \\
& & & \textcolor{blue}{+ TT-SAC} (w/ motion)
  & \textbf{6.5690} & \textbf{7.7912} & 0.9962 & 0.1323 & 0.8349 & 26.4608 & 100.4395 \\
\hline

% ========================
% CelebV-HQ 大组（共 16 行）
% ========================
\multirow{16}{*}{\textbf{CelebV-HQ}}

 & \multicolumn{3}{c}{\textit{Real video}} 
 & 2.2966 & 9.7919 & 0.9938 & 0      & 1      & 0    & 0 \\
 \addlinespace[0.1ex]  
\cline{2-11}
\addlinespace[0.1ex]  
& \multirow{3}{*}{AniTalker~\cite{liu2024anitalker}}
   & \multirow{3}{*}{ACMMM 2024} & Baseline
   & 2.1208 & 9.9276 & 0.9949 & 0.2762 & 0.6604 & 72.6029 & 370.0368 \\
 &  & & \textcolor{blue}{+ TT-SAC}
   & 1.6679 & 10.5478 & 0.9952 & 0.2495 & 0.6882 & 60.9655 & 315.9870 \\
 &  & & \textcolor{blue}{+ TT-SAC} (w/ motion)
   & \textbf{2.2708} & \textbf{9.7044} & \textbf{0.9956} & \textbf{0.1627} & \textbf{0.7772} & \textbf{44.2022} & \textbf{215.8235} \\
 \addlinespace[0.1ex]  
 \cline{2-11}
 \addlinespace[0.1ex]  
% FLOAT ------------------------------------------------------
 & \multirow{3}{*}{FLOAT~\cite{ki2025float}}
   & \multirow{3}{*}{ICCV 2025} & Baseline
   & 2.3673 & 9.6682 & 0.9948 & 0.2746 & 0.6373 & 58.1084 & 363.1891 \\
 &  & & \textcolor{blue}{+ TT-SAC}
   & 2.7147 & 9.3903 & 0.9954 & \textbf{0.1979} & 0.7050 & \textbf{39.2848} & 243.2008 \\
 &  & & \textcolor{blue}{+ TT-SAC} (w/ motion)
   & \textbf{2.7475} & \textbf{9.3606} & \textbf{0.9957} & \textbf{0.1979} & \textbf{0.7141} & 39.6612 & \textbf{243.1321} \\
 \addlinespace[0.1ex]  
 \cline{2-11}
 \addlinespace[0.1ex]  
% JoyVASA -----------------------------------------------------
 & \multirow{3}{*}{JoyVASA~\cite{cao2024joyvasa}}
   & \multirow{3}{*}{arXiv 2024} & Baseline
   & 2.7024 & 9.9814 & 0.9961 & 0.1432 & 0.7934 & 26.4852 & 272.2531 \\
 &  & & \textcolor{blue}{+ TT-SAC}
   & 2.3604 & 9.8732 & \textbf{0.9962} & 0.0819 & 0.8527 & \textbf{17.5746} & 171.9289 \\
 &  & & \textcolor{blue}{+ TT-SAC} (w/ motion)
   & \textbf{2.7530} & \textbf{9.5639} & 0.9961 & \textbf{0.0794} & \textbf{0.8680} & 16.7644 & \textbf{164.6510} \\
 \addlinespace[0.1ex]  
 \cline{2-11}
 \addlinespace[0.1ex]  
% Sadtalker --------------------------------------------------
 & \multirow{3}{*}{SadTalker~\cite{zhang2023sadtalker}}
   & \multirow{3}{*}{CVPR 2023} & Baseline
   & 2.9788 & 9.3258 & 0.9960 & 0.1540 & 0.7533 & 50.8880 & 351.6044 \\
 &  & & \textcolor{blue}{+ TT-SAC}
   & 2.9844 & 9.3918 & \textbf{0.9960} & \textbf{0.0922} & \textbf{0.8419} & \textbf{28.5864} & \textbf{205.7255} \\
 &  & & \textcolor{blue}{+ TT-SAC} (w/ motion)
   & \textbf{3.0859} & \textbf{9.1911} & 0.9957 & 0.0990 & 0.8116 & 35.0428 & 212.1751 \\
 \addlinespace[0.1ex]  
\cline{2-11}
\addlinespace[0.1ex]  
% Sonic ------------------------------------------------------
 & \multirow{3}{*}{Sonic~\cite{ji2025sonic}}
   & \multirow{3}{*}{CVPR 2025} & Baseline
   & 3.0684 & 9.0353 & 0.9963 & 0.1874 & 0.7705 & 29.9165 & 245.1475 \\
 &  & & \textcolor{blue}{+ TT-SAC}
   & \textbf{3.0803} & \textbf{9.0163} & 0.9963 & \textbf{0.0989} & \textbf{0.8664} & \textbf{17.0924} & \textbf{152.4971} \\
 &  & &  \textcolor{blue}{+ TT-SAC} (w/ motion)
   & 2.9029 & 9.0855 & \textbf{0.9966} & 0.1384 & 0.8139 & 24.0233 & 179.9765 \\

\hline

% ========================
% RAVDESS 大组（共 16 行）
% ========================
\multirow{16}{*}{\textbf{RAVDESS}}

 & \multicolumn{3}{c}{\textit{Real video}}
 & 2.8529 & 7.6629 & 0.9948 & 0      & 1      & 0   & 0 \\
 \addlinespace[0.1ex]  
 \cline{2-11}
 \addlinespace[0.1ex]  
% AniTalker ---------------------------------------------------
 & \multirow{3}{*}{AniTalker~\cite{liu2024anitalker}}
   & \multirow{3}{*}{ACMMM 2024} & Baseline
   & 1.7118 & 8.4952 & 0.9953 & 0.1580 & 0.8890 & 38.3017 & 95.9668 \\
 &  & & \textcolor{blue}{+ TT-SAC}
   & \textbf{1.8344} & \textbf{8.3215} & 0.9954 & 0.1105 & 0.9227 & 21.4968 & 73.4484 \\
 &  & & \textcolor{blue}{+ TT-SAC} (w/ motion)
   & 1.7437 & 8.5157 & \textbf{0.9956} & \textbf{0.0889} & \textbf{0.9269} & \textbf{21.28172} & \textbf{59.8016} \\
 \addlinespace[0.1ex]  
\cline{2-11}
\addlinespace[0.1ex]  
% FLOAT ------------------------------------------------------
 & \multirow{3}{*}{FLOAT~\cite{ki2025float}}
   & \multirow{3}{*}{ICCV 2025} & Baseline
   & 3.4310 & 6.9730 & 0.9951 & 0.0992 & 0.8809 & 9.5785 & 88.6955 \\
 &  & & \textcolor{blue}{+ TT-SAC}
   & \textbf{3.5260} & 6.9342 & 0.9955 & \textbf{0.0678} & \textbf{0.9107} & \textbf{7.49182} & 70.6535 \\
 &  & & \textcolor{blue}{+ TT-SAC} (w/ motion)
   & 3.5004 & \textbf{6.9190} & \textbf{0.9956} & 0.0727 & 0.9066 & 7.9720 & \textbf{68.3226} \\
 \addlinespace[0.1ex]  
\cline{2-11}
\addlinespace[0.1ex]  
 % JoyVASA -----------------------------------------------------
 & \multirow{3}{*}{JoyVASA~\cite{cao2024joyvasa}}
   & \multirow{3}{*}{arXiv 2024} & Baseline
   & 1.6303 & 9.4236 & 0.9958 & 0.0671 & 0.8844 & 10.0316 & 108.1995 \\
 &  & & \textcolor{blue}{+ TT-SAC}
   & 1.2794 & 9.7171 & 0.9958 & 0.0431 & 0.9049 & 10.0169 & 94.4847 \\
 & & & \textcolor{blue}{+ TT-SAC} (w/ motion)
   & \textbf{1.7931} & \textbf{9.4036} & \textbf{0.9959} & \textbf{0.0409} & \textbf{0.9184} & \textbf{6.86062} & \textbf{69.9282} \\
 \addlinespace[0.1ex]  
\cline{2-11}
\addlinespace[0.1ex]  
% Sadtalker --------------------------------------------------
 & \multirow{3}{*}{SadTalker~\cite{zhang2023sadtalker}}
   & \multirow{3}{*}{CVPR 2023} & Baseline
   & 1.9095 & 8.1480 & \textbf{0.9957} & 0.0821 & 0.8141 & 20.9527 & 96.0455 \\
 &  & & \textcolor{blue}{+ TT-SAC}
   & \textbf{1.9600} & \textbf{8.1085} & 0.9951 & \textbf{0.0614} & \textbf{0.8815} & \textbf{11.77912} & \textbf{60.6654} \\
 &  & & \textcolor{blue}{+ TT-SAC} (w/ motion)
   & 1.8419 & 8.3261 & 0.9955 & 0.0655 & 0.8532 & 16.1888 & 77.9732 \\
 \addlinespace[0.1ex]  
 \cline{2-11}
 \addlinespace[0.1ex]  
% Sonic ------------------------------------------------------
 & \multirow{3}{*}{Sonic~\cite{ji2025sonic}}
   & \multirow{3}{*}{CVPR 2025} & Baseline
   & 2.5563 & 7.7422 & 0.9961 & 0.1246 & 0.8986 & 10.5022 & 63.1852 \\
 &  & & \textcolor{blue}{+ TT-SAC}
   & \textbf{2.5648} & 7.7444 & 0.9961 & \textbf{0.0647} & \textbf{0.9432} & \textbf{5.87422} & \textbf{36.0834} \\
 &  & & \textcolor{blue}{+ TT-SAC} (w/ motion)
   & 2.4840 & \textbf{7.6838} & \textbf{0.9963} & 0.0896 & 0.9280 & 8.2122 & 48.1288 \\

\bottomrule
\end{tabular}
}
\label{tab:main}
\end{table*}

\subsection{Evaluation}

Table~\ref{tab:main} summarises the quantitative results, and Figure~\ref{fig:main} presents visual comparisons. Below, we provide a detailed analysis and evaluation.

\textbf{Lip-audio synchronization.} Across all datasets and generators, TT-SAC consistently improves lip-audio alignment, as measured by Sync-C and Sync-D. On the Hallo dataset, AniTalker increases from 3.9164 to 4.0822 in Sync-C using identity-only TT-SAC, with refinement to 3.9488 when motion pathways are included. 
On RAVDESS, Sonic improves from 2.5563 to 2.5648 in Sync-C, while Sync-D decreases correspondingly. These improvements indicate that dynamically refining the reference reduces identity drift and ensures that the mouth region remains aligned with the input speech signal over time. % The motion pathway extension is effective for generators such as AniTalker and JoyVASA, which exhibit more pronounced motion-driven misalignments. 
By using the generated frames themselves to update the latent conditioning, TT-SAC addresses a fundamental limitation of prior methods that rely on a static reference.

\textbf{Temporal smoothness} quantified by Smooth$\uparrow$ benefits from TT-SAC. FLOAT on Hallo improves from 0.9946 to 0.9955 with identity-only TT-SAC and remains stable at 0.9952 when motion TT-SAC is applied. Improvements for JoyVASA are smaller, reflecting its already smooth baseline generation. These results confirm that computing a self-consistent conditioning feature from early frames reduces stochastic fluctuations, as predicted by our theoretical variance analysis. By mitigating frame-to-frame jitter, TT-SAC enhances perceptual realism and contributes to smoother video transitions.

\textbf{Perceptual quality.} Perceptual quality (LPIPS$\downarrow$) improves significantly across all datasets and models when applying TT-SAC. AniTalker on Hallo sees a reduction from 0.2762 to 0.2350 with identity TT-SAC and further to 0.1561 when motion pathways are incorporated. Sonic on RAVDESS shows an LPIPS decrease from 0.1246 to 0.0647 using identity-only TT-SAC. These results indicate that TT-SAC preserves high-frequency facial details and textures, leading to perceptually more realistic frames. Motion pathway refinement further enhances quality when the generator supports temporally coherent motion, demonstrating the flexibility of TT-SAC to improve both appearance and dynamics.

\textbf{Identity preservation.} The CSIM metric shows substantial improvements in identity consistency when using TT-SAC. For instance, AniTalker on Hallo increases from 0.7561 to 0.7990 by applying identity-only TT-SAC, rising further to 0.8445 with motion-aware TT-SAC. Similarly, JoyVASA on CelebV-HQ improves from 0.7934 to 0.8527. These results demonstrate that TT-SAC reduces identity drift by refining the conditioning feature toward a self-consistent fixed point. Motion-pathway updates can introduce small trade-offs in some cases, such as JoyVASA on CelebV-HQ, but the overall identity consistency is substantially improved. The gains are consistent with the fixed-point self-consistency formulation and the variance-reduction analysis in Section~\ref{sec:theory}, providing theoretical justification for TT-SAC’s effectiveness.

% \begin{figure*}[tbp]
% \centering
% \subfloat[Sync-C$\uparrow$]{%
% \includegraphics[width=0.32\linewidth, height=3.5cm]{figures/subplot_metric_1.pdf}%
% \label{fig:sync-c}}
% % \hfill
% \subfloat[Sync-D$\downarrow$]{%
% \includegraphics[width=0.24\linewidth, height=3.5cm]{figures/subplot_metric_2.pdf}%
% \label{fig:sync-d}}
% % \hfill
% \subfloat[CSIM$\uparrow$]{%
% \includegraphics[width=0.21\linewidth, height=3.5cm]{figures/subplot_metric_3.pdf}%
% \label{fig:csim}}
% % \hfill
% \subfloat[FID$\downarrow$]{%
% \includegraphics[width=0.195\linewidth, height=3.5cm]{figures/subplot_metric_4.pdf}%
% \label{fig:fid}}

% \caption{Evaluation of the aggregation window size $K$ across five audio-driven talking-head models using multiple metrics. Red dots indicate the baseline without TT-SAC ($K=0$), while blue dots show performance for different $K$ values ($K=1,2,3,5,10$) with TT-SAC. The plots demonstrate that TT-SAC consistently improves all metrics, with optimal performance typically achieved at smaller $K$ values, balancing temporal stability and identity preservation.}

% \label{fig:hyper-eval}
% \end{figure*}

\begin{figure*}[tbp]
\centering
\includegraphics[width=0.12\textwidth]{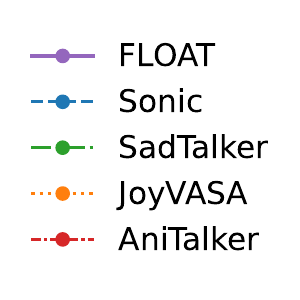}% \hfill
\subfloat[Sync-C$\uparrow$]{%
\includegraphics[width=0.22\linewidth]{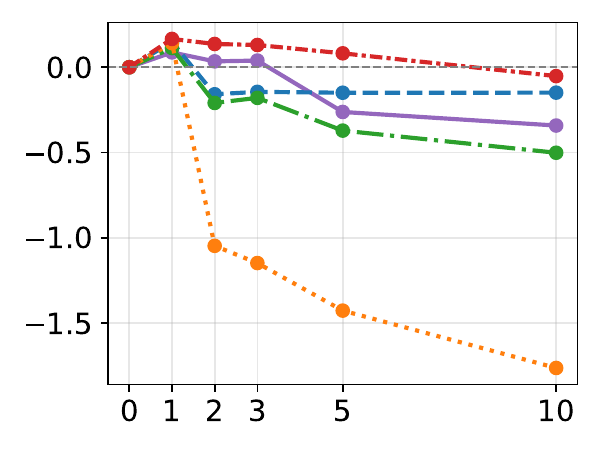}%
}% \hfill
\subfloat[Sync-D$\downarrow$]{%
\includegraphics[width=0.22\linewidth]{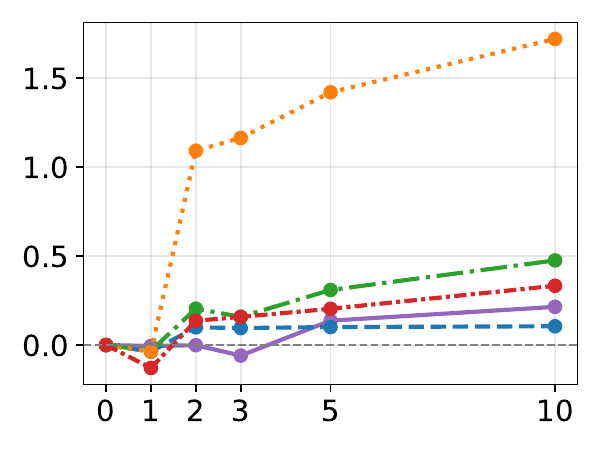}%
}% \hfill
\subfloat[CSIM$\uparrow$]{%
\includegraphics[width=0.22\linewidth]{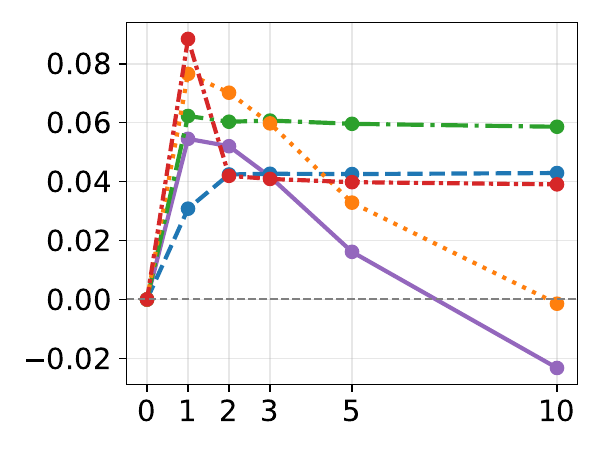}%
}% \hfill
\subfloat[FID$\downarrow$]{%
\includegraphics[width=0.22\linewidth]{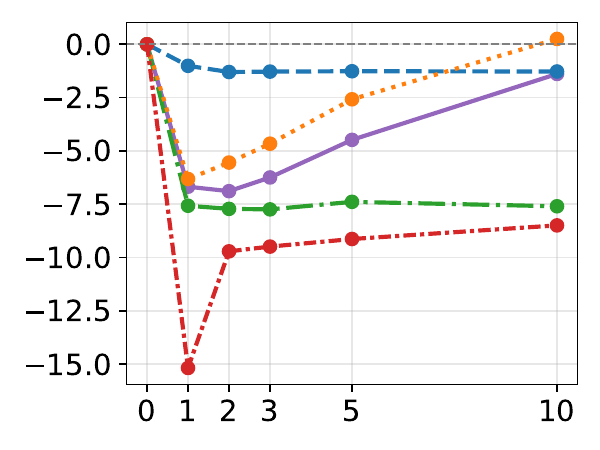}%
}
\caption{Effect of the number of aggregated frames $K$ in TT-SAC across five models.
The horizontal axis shows $K$, while the vertical axis reports performance change relative to the baseline ($\Delta$), where $K=0$ denotes standard inference.
For metrics with $\uparrow$ (Sync-C, CSIM), higher values indicate improvement, whereas for $\downarrow$ (Sync-D, FID), lower values indicate improvement.
Smaller $K$ generally yields larger gains, while larger $K$ leads to saturation or slight degradation.}
% Effect of aggregation window size $K$ in TT-SAC across five models. 
% The horizontal axis denotes the aggregation window size $K$, and the vertical axis shows performance change relative to baseline ($\Delta$), where $0$ corresponds to standard inference without TT-SAC ($K\!=\!0$). 
% For metrics with $\uparrow$ (Sync-C, CSIM), positive values indicate improvement, whereas for metrics with $\downarrow$ (Sync-D, FID), negative values indicate improvement. 
% All subfigures share the same legend. 
% Smaller window sizes yield stronger gains, while larger $K$ values show saturation or slight degradation, reflecting a bias-variance trade-off.
\label{fig:hyper-eval}
\end{figure*}

% \begin{figure*}[tbp]
% \centering
% \begin{subfigure}[t]{0.32\linewidth}
% \centering\includegraphics[width=\linewidth, height=3.5cm]{figures/subplot_metric_1.pdf}
% \caption{Sync-C$\uparrow$}\label{fig:sync-c}
% \end{subfigure}\hfill
% \begin{subfigure}[t]{0.24\linewidth}
% \centering\includegraphics[width=\linewidth, height=3.5cm]{figures/subplot_metric_2.pdf}
% \caption{Sync-D$\downarrow$}\label{fig:sync-d}
% \end{subfigure}\hfill
% \begin{subfigure}[t]{0.21\linewidth}
% \centering\includegraphics[width=\linewidth, height=3.5cm]{figures/subplot_metric_3.pdf}
% \caption{CSIM$\uparrow$}\label{fig:csim}
% \end{subfigure}\hfill
% \begin{subfigure}[t]{0.195\linewidth}
% \centering\includegraphics[width=\linewidth, height=3.5cm]{figures/subplot_metric_4.pdf}
% \caption{FID$\downarrow$}\label{fig:fid}
% \end{subfigure}
% \caption{Evaluation of the aggregation window size $K$ across five audio-driven talking-head models using multiple metrics. Red dots indicate the baseline without TT-SAC ($K=0$), while blue dots show performance for different $K$ values ($K=1,2,3,5,10$) with TT-SAC. The plots demonstrate that TT-SAC consistently improves all metrics, with optimal performance typically achieved at smaller $K$ values, balancing temporal stability and identity preservation.
% }
% \label{fig:hyper-eval}
% \end{figure*}

\textbf{Video-level fidelity.} FVD captures holistic video quality, encompassing both temporal and perceptual aspects. Across all generators, TT-SAC consistently reduces FVD, indicating improved video realism. AniTalker on Hallo decreases from 143.40 to 121.77 with identity-only TT-SAC and further to 85.11 with motion-aware TT-SAC. Sonic on RAVDESS sees nearly a 50\% improvement, from 63.19 to 36.08. These results demonstrate that TT-SAC’s two-pass refinement reduces cumulative errors and stabilizes the generated sequence. Extending the refinement to motion pathways amplifies these gains, especially in datasets with expressive facial dynamics, highlighting the model-agnostic nature of the approach.

\subsection{Discussion}

\textbf{Dataset- and model-specific observations.} On Hallo, TT-SAC achieves the largest gains in LPIPS and CSIM, reflecting its effectiveness in stabilizing highly expressive motions that cause identity drift. CelebV-HQ shows mixed trade-offs between motion and identity: for JoyVASA, motion refinement slightly reduces identity while improving temporal smoothness, highlighting the importance of selecting an appropriate aggregation window $K$. On RAVDESS, improvements are uniformly strong across all metrics, as the controlled emotional expressions create highly correlated frames, which TT-SAC exploits for variance reduction. FLOAT and Sonic show more modest lip-sync improvements, consistent with their strong baseline performance, yet they still benefit from reduced frame-to-frame variance.

% \textbf{Effect of aggregation window size.}
% Fig.~\ref{fig:hyper-eval} reports performance change relative to the baseline ($\Delta$), where the dashed horizontal line at $0$ corresponds to standard inference without TT-SAC ($K=0$). For metrics with $\uparrow$ (Sync-C, CSIM), positive values indicate improvement, whereas for metrics with $\downarrow$ (Sync-D, FID), negative values indicate improvement.

% Across models, introducing TT-SAC ($K\!\ge\! 1$) generally shifts performance away from zero in the favorable direction. For Sync-C and CSIM, most models exhibit positive gains at small $K$, indicating improved lip-audio alignment and stronger identity consistency. For Sync-D and FID, improvements appear as negative deltas, particularly at moderate $K$, reflecting reduced lip discrepancy and enhanced perceptual quality.
% % 
% Notably, smaller aggregation windows ($K\!=\!1$ or $2$) often yield the largest improvements relative to baseline. As $K$ increases, gains tend to saturate or partially diminish for some models, suggesting that excessive aggregation may introduce bias by over-smoothing dynamic facial variations. This trend aligns with the bias-variance trade-off implied by our theoretical analysis: moderate $K$ effectively reduces stochastic variance while preserving expressive detail.

\textbf{Effect of the number of aggregated frames.} Fig. \ref{fig:hyper-eval} reports performance change relative to the baseline ($\Delta$), where the dashed horizontal line at $0$ corresponds to standard inference without TT-SAC ($K=0$). 
For metrics with $\uparrow$ (Sync-C, CSIM), higher values indicate improvement, whereas for $\downarrow$ (Sync-D, FID), lower values indicate improvement.

Across models, enabling TT-SAC ($K\!\ge\!1$) generally shifts performance in the favorable direction. 
Small values of $K$ consistently improve lip-audio synchronization (Sync-C) and identity similarity (CSIM), while reducing lip discrepancy (Sync-D) and perceptual error (FID).
The largest gains typically occur at small $K$ (\eg, $K\!=\!1$ or $2$). 
As $K$ increases, improvements tend to saturate or slightly diminish, suggesting that excessive aggregation may over-smooth dynamic facial variations. 
This behavior is consistent with the bias-variance trade-off discussed in our theoretical analysis.

% Overall, these results validate the robustness of TT-SAC across models and datasets, and provide practical guidance for selecting $K$ to balance temporal coherence and expressiveness.

% \textbf{Motion pathway refinement.} To assess the effect of motion refinement in TT-SAC, we compare identity-only conditioning refinement (+ TT-SAC) with the motion-aware variant (+ TT-SAC (w/ motion)), which additionally aggregates motion-related pathways (\eg, keypoint streams or flow fields) when supported by the model.
% The impact of motion refinement is model- and dataset-dependent. On expressive datasets such as Hallo and CelebV-HQ, incorporating motion often yields further improvements in perceptual realism (FID/FVD) and distortion (LPIPS), as seen for AniTalker and FLOAT. However, gains in lip-audio synchronization (Sync-C/Sync-D) and temporal smoothness (Smooth) are not uniform, and in some cases identity-only refinement achieves stronger alignment metrics. 
% Identity consistency (CSIM) is generally preserved or improved relative to baseline, though motion-aware refinement can introduce minor trade-offs for certain model-dataset combinations (\eg, slight CSIM reductions compared to identity-only refinement). 

\begin{figure*}[tbp]
\centering
% \vspace{-0.5cm}
\subfloat[Human half-body]{%
\includegraphics[width=0.44\linewidth]{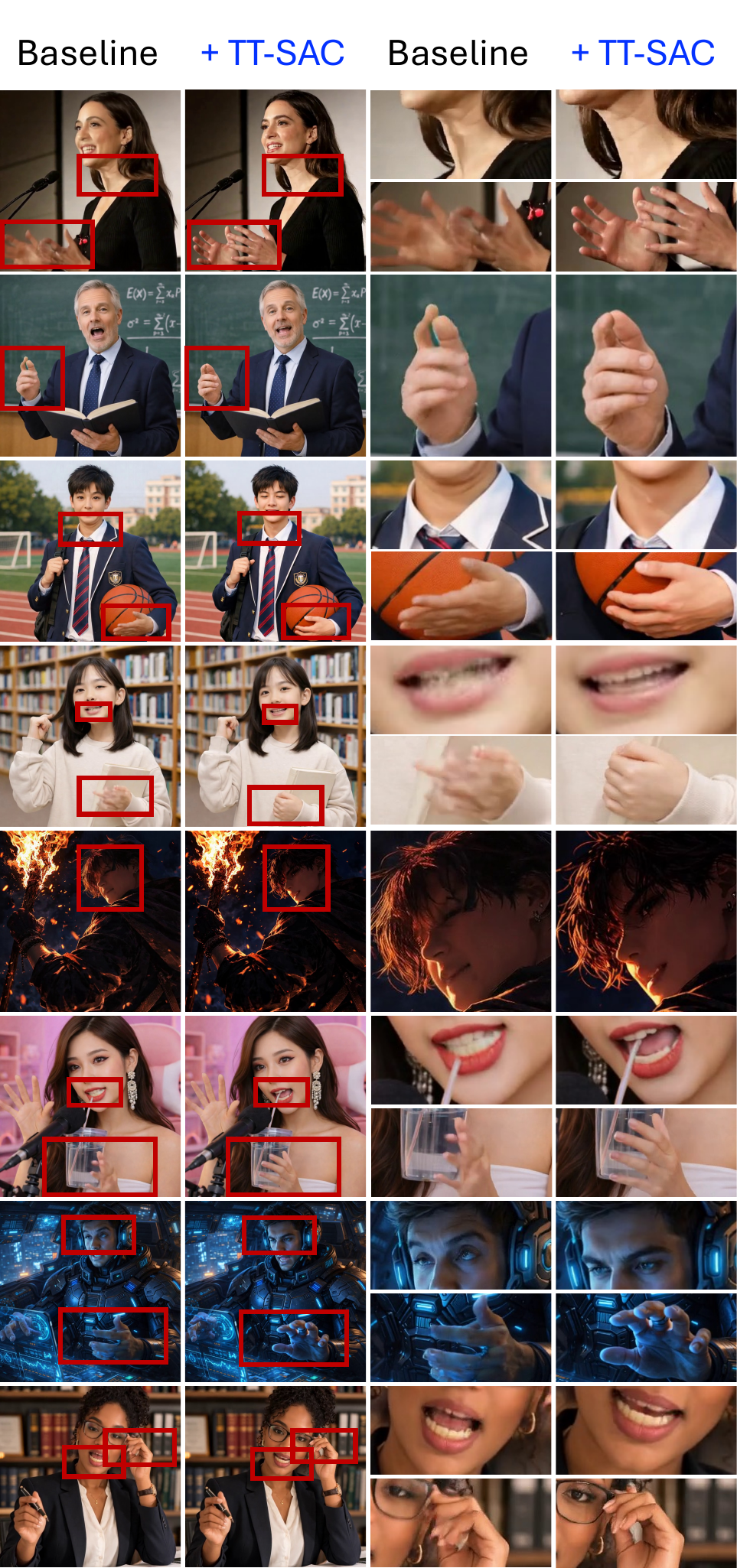}
}
\hfill
\subfloat[Avatar]{%
\includegraphics[width=0.44\linewidth]{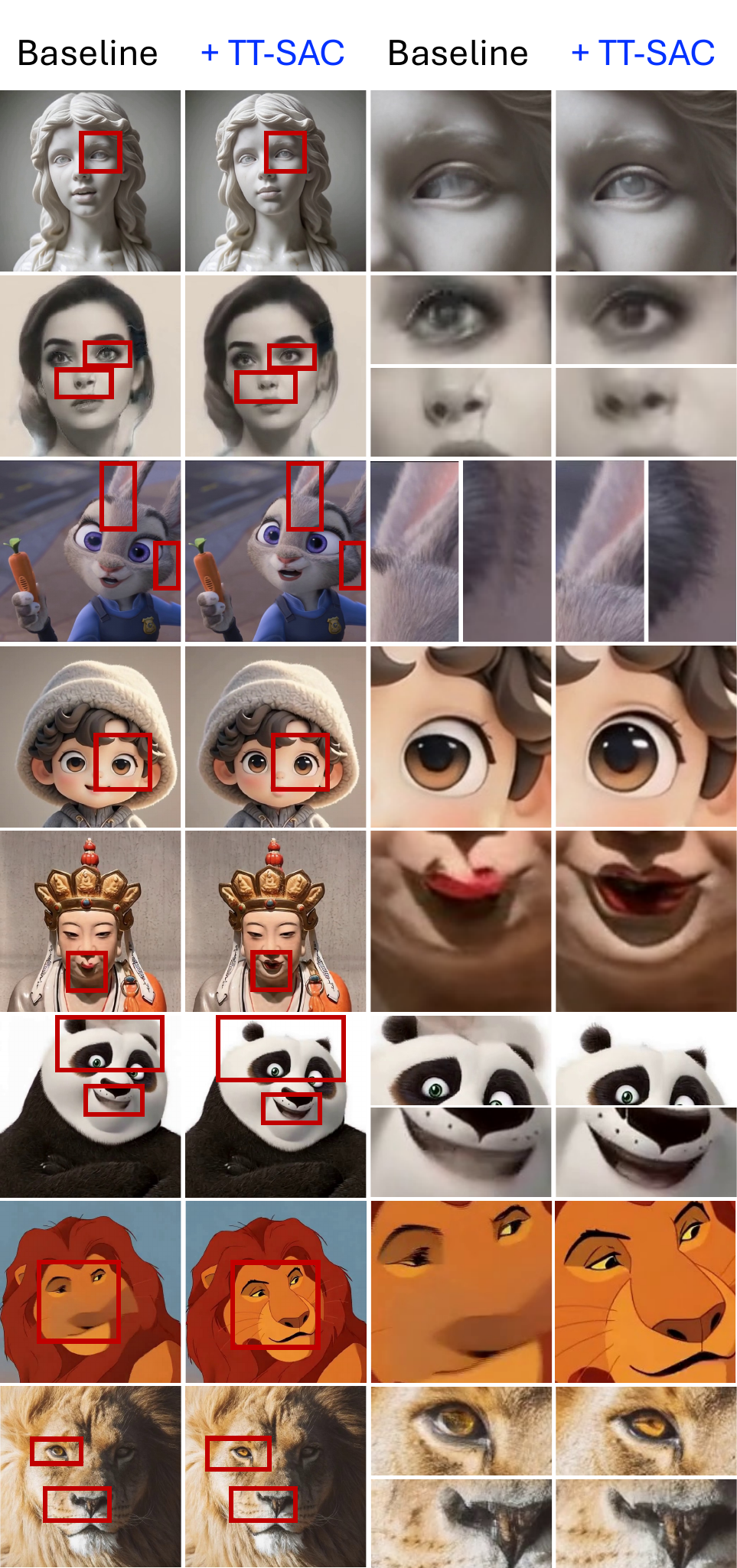}
}
\caption{Generalization to other audio-driven video synthesis tasks in OmniAvatar \cite{gan2025omniavatar}. % The top row shows baseline results, while the bottom row shows results with TT-SAC. 
Our method improves motion coherence and enhances eye dynamics and fine-grained details in both human half-body generation and avatar animation tasks, demonstrating strong cross-task generalization.}
\label{fig:two_compare}
\end{figure*}

\textbf{Motion pathway refinement.} We further analyze the effect of incorporating motion cues in TT-SAC by comparing identity-only refinement (+TT-SAC) with the motion-aware variant (+TT-SAC (w/ motion)), which additionally aggregates motion-related pathways (\eg, keypoint streams or motion fields).
Quantitatively, the impact of motion refinement varies across models and datasets. On highly expressive datasets such as Hallo and CelebV-HQ, incorporating motion often improves perceptual realism and temporal coherence, leading to better FID/FVD and LPIPS scores for models such as AniTalker and FLOAT. However, improvements in synchronization (Sync-C/Sync-D) and temporal smoothness (Smooth) are not always consistent, and identity-only TT-SAC occasionally achieves slightly stronger alignment metrics.
Qualitatively, motion-aware TT-SAC effectively suppresses temporal artifacts commonly observed in baseline models, including sudden facial shifts, frame-to-frame jitter, and structural instability (Fig. \ref{fig:main}). These improvements are particularly visible in regions with large motion dynamics, where identity-only TT-SAC already stabilizes facial appearance while motion-aware TT-SAC further reduces temporal deviations. Identity consistency (CSIM) remains largely preserved relative to baseline, although minor trade-offs can occur for certain model-dataset combinations.

\textbf{Failure modes and generality.} 
TT-SAC assumes that the initial generation provides sufficiently informative identity cues. 
If early frames exhibit severe identity degradation, the Monte Carlo estimate may inherit this bias. 
However, aggregation is performed in encoder feature space rather than pixel space, and identity-focused encoders tend to suppress transient artifacts, mitigating error amplification. 
In practice, the first few frames typically retain reliable identity information, making the estimator stable; extreme motion or severely degraded initial outputs may reduce the magnitude of improvement but do not destabilize the process.

\textbf{Applicability to other video generation tasks.} Our method is not tied to a specific talking-head architecture. Since TT-SAC operates purely on the conditioning representation at test time, it can be directly applied to other audio-driven video generation pipelines without modifying model parameters.
To demonstrate this generality, we apply TT-SAC to OmniAvatar \cite{gan2025omniavatar}, which supports more challenging synthesis scenarios such as human half-body generation and non-human avatar animation. As shown in Fig.~\ref{fig:two_compare}, TT-SAC consistently improves temporal stability and visual fidelity in both settings. In half-body synthesis, the method produces more coherent hand and body motion, while in avatar animation it enhances eye dynamics and fine structural details.
These results indicate that TT-SAC generalizes beyond face-only talking-head generation and can serve as a test-time conditioning adaptation for diverse audio-driven video synthesis tasks.

\section{Conclusion}

We presented TT-SAC, a conditioning-level test-time adaptation framework for audio-driven talking-head generation. Instead of relying solely on a static reference feature, TT-SAC performs feature aggregation and a single self-consistency update to project the conditioning representation toward a generator-encoder equilibrium.
Extensive experiments across multiple datasets and pretrained generators demonstrate consistent improvements in identity preservation, perceptual quality, and video-level realism, with dataset- and model-dependent effects on lip synchronization and temporal smoothness. Motion-aware TT-SAC further enhances stability for several architectures, while preserving alignment and identity in most settings.
Our theoretical analysis provides a principled interpretation of TT-SAC through variance reduction and local fixed-point contraction, explaining why small aggregation windows achieve strong empirical performance. Since TT-SAC requires no retraining, architectural modification, or gradient-based adaptation, it serves as a practical and broadly applicable test-time stabilization strategy.
We believe this work highlights conditioning self-consistency as a general mechanism for improving stability and coherence in generative video models.

% \section{Conclusion}

% We introduced TT-SAC, a principled inference-time framework that enhances audio-driven portrait animation by replacing the conventional static reference with a self-consistent, temporally adaptive one.  
% Through lightweight feature aggregation and a single feedback pass, TT-SAC significantly improves identity stability, expression coherence, and lip-audio alignment across diverse pretrained generators.  
% Our theoretical analysis establishes variance reduction, fixed-point convergence, and an optimal aggregation window, providing a unified explanation for the empirical gains.  
% Because TT-SAC requires no retraining, modifies no architecture, and incurs negligible computational overhead, it offers a practical drop-in correction applicable to diffusion-, flow-, and keypoint-based models.  
% We hope this work motivates further exploration of self-consistency as a general principle for stabilizing generative video models.

\bibliographystyle{IEEEtran}
\bibliography{research}

% \newpage

\vskip -1.5\baselineskip plus -1fil

\begin{IEEEbiography}[{\includegraphics[trim=40 35 40 16,width=1in,height=1.25in,clip,keepaspectratio]{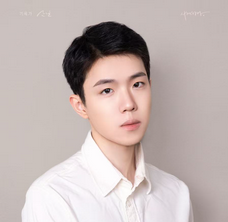}}]{Zhicheng Zhang} is a Ph.D. student at the University of New South Wales (UNSW), Australia, supervised by Dr. Yu Zhang (2024-present). He received his M.S. from The University of Queensland, Australia (2022-2023). He is currently a visiting scholar at the ARC Research Hub hosted by Griffith University, under the supervision of Dr. Lei Wang and Prof. Yongsheng Gao. He also serves as the workshop coordinator for the TIME 2026 workshop, organized as part of The Web Conference 2026 (WWW 2026). His research interests include talking-head generation, temporal modeling, and computer vision.
\end{IEEEbiography}

\vskip -1.5\baselineskip plus -1fil

\begin{IEEEbiography}
[{\includegraphics[width=1in,height=1.25in,clip,keepaspectratio]{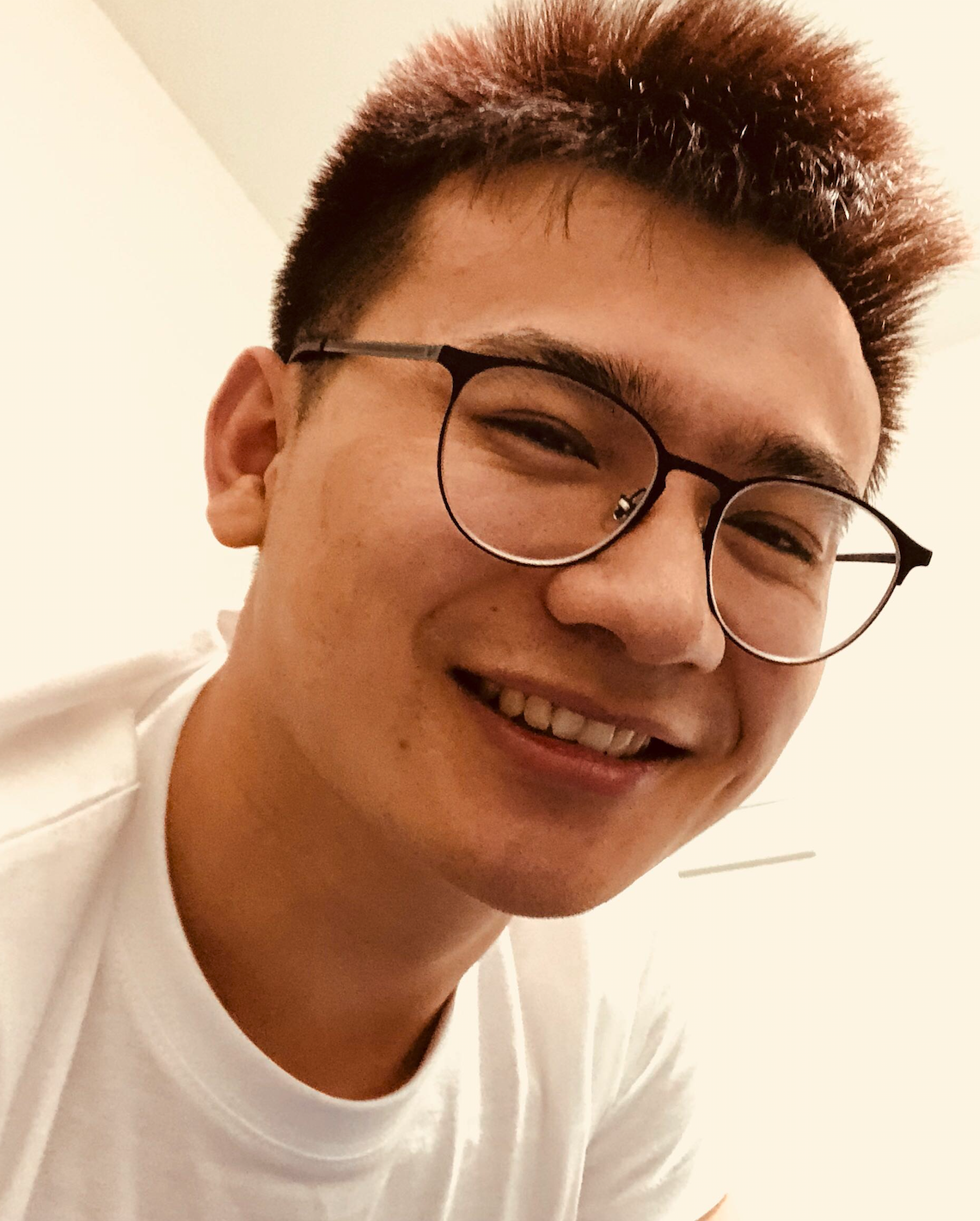}}]{Lei Wang} received his M.E. in Software Engineering from the University of Western Australia (UWA) in 2018 and his Ph.D. in Engineering and Computer Science from the Australian National University (ANU) in 2023. He is a Research Fellow in the School of Electrical and Electronic Engineering at Griffith University and a Visiting Scientist with Data61/CSIRO. He leads the Temporal Intelligence and Motion Extraction (TIME) Lab at Griffith University. He previously held research positions at ANU, UWA, and Data61/CSIRO. His research focuses on motion-, data-, and model-centric approaches to video action recognition and anomaly detection. He has authored numerous first-author papers in top-tier venues, including CVPR, ICCV, ECCV, ACM Multimedia, NeurIPS, ICLR, ICML, AAAI, TPAMI, IJCV, and TIP, and received the Sang Uk Lee Best Student Paper Award at ACCV 2022. He serves as an Area Chair for ACM Multimedia 2024-2025, ICASSP 2025, and ICPR 2024, and was recognized as an Outstanding Area Chair at ACM Multimedia 2024.
\end{IEEEbiography}

\vskip -1.5\baselineskip plus -1fil

\begin{IEEEbiography}[{\includegraphics[trim=140 190 110 35, width=1in,height=1.25in,clip,keepaspectratio]{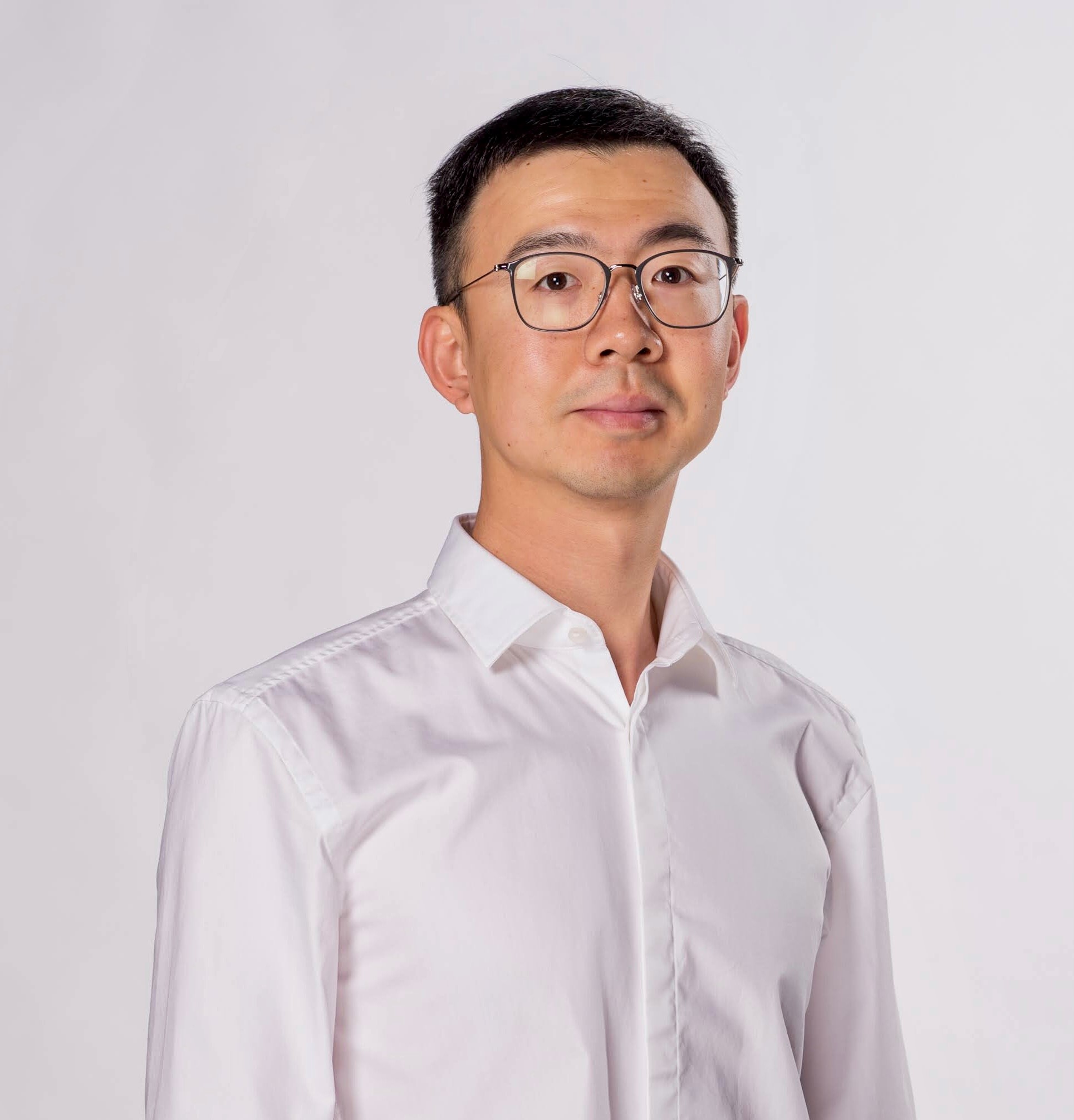}}]{Yu Zhang} is a Lecturer in Data Science at the School of Business, UNSW. His research focuses on machine learning for information and knowledge management, graph representation learning and heterogeneous network analysis, text and data mining for asset management, AI-driven industrial systems, sustainable logistics, and net-zero energy solutions. He has led and contributed to multiple government- and industry-funded projects on trustworthy AI for battery health monitoring, AI-enabled defense logistics, cyber threat intelligence, and hybrid energy systems for net-zero buildings. His work has been published in leading journals and conferences, including Information Processing \& Management, Journal of Informetrics, AAAI, CIKM, and PAKDD. He has received several Best/Excellent Paper Awards, including ADMA 2024, ICEBE 2024, and VICFCNT 2020.
\end{IEEEbiography}

\vskip -1.5\baselineskip plus -1fil

\begin{IEEEbiography}[{\includegraphics[trim=40 50 30 5, width=1in,height=1.25in,clip,keepaspectratio]{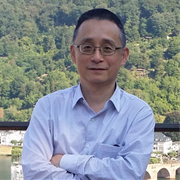}}]{Yongsheng Gao} received the BSc and MSc degrees in Electronic Engineering from Zhejiang University, China, in 1985 and 1988, respectively, and the PhD degree in Computer Engineering from Nanyang Technological University, Singapore. He is currently a Professor at the School of Engineering and Built Environment, Griffith University, and Director of the ARC Research Hub for Driving Farming Productivity and Disease Prevention, Australia. He was previously the Leader of the Biosecurity Group at the Queensland Research Laboratory, National ICT Australia (ARC Centre of Excellence), a consultant at Panasonic Singapore Laboratories, and an Assistant Professor at Nanyang Technological University. His research interests include smart farming, machine vision for agriculture, biosecurity, face recognition, biometrics, image retrieval, computer vision, pattern recognition, environmental informatics, and medical imaging. He is a recipient of the 2025 ARC Industry Laureate Fellow.\end{IEEEbiography}

\end{document}